\documentclass[journal]{IEEEtran}
\newtheorem{theorem}{Theorem}[section]

\newtheorem{proposition}[theorem]{Proposition}

\newenvironment{proof}[1][Proof.]{\begin{trivlist}
\item[\hskip \labelsep { #1}]}{\end{trivlist}}
\newcommand{\qed}{\nobreak \ifvmode \relax \else
      \ifdim\lastskip<1.5em \hskip-\lastskip
      \hskip1.5em plus0em minus0.5em \fi \nobreak
      \vrule height0.75em width0.5em depth0.25em\fi}

\ifCLASSINFOpdf
   \usepackage[pdftex]{graphicx}
\else
\fi
\usepackage{etoolbox}

\usepackage[cmex10]{amsmath}
\usepackage{amssymb}
\usepackage{notacion2}
\usepackage[normalem]{ulem}
\newtheorem{definition}{Definition}[section]

\ifx\pdfoutput\undefined
\usepackage{graphicx}
\else
\usepackage[pdftex]{graphicx}
\usepackage{epstopdf}
\fi
\usepackage[svgnames, x11names]{xcolor}
\usepackage{environ}

\NewEnviron{equationw}{
\begin{equation}\begin{split}
\BODY
\end{split}\end{equation}
}

  \usepackage[caption=false,font=normalsize,labelfont=sf,textfont=sf]{subfig}

\begin{document}
\title{Complex-Valued Kernel Methods for Regression }

\author{Rafael~Boloix-Tortosa, %~\IEEEmembership{Member,~IEEE,}
       Juan~Jos\'e~Murillo-Fuentes, %~\IEEEmembership{Senior~Member,~IEEE,}
       % F.~Javier~Pay\'an-Somet,
        Irene Santos Vel\'azquez,
        and Fernando P\'erez-Cruz, %~\IEEEmembership{Senior~Member,~IEEE,}% <-this % stops a space
\thanks{R. Boloix-Tortosa, %F.~Javier~Pay\'an-Somet, 
        Irene Santos-Vel\'azquez,
        and Juan~Jos\'e~Murillo-Fuentes are with the Department of Signal Theory and Communications, University of Seville, Spain. e-mail: rboloix@us.es. Fernando P\'erez-Cruz is with Department of Signal Theory and Communications, University Carlos III de Madrid, Spain}% <-this % stops a space
\thanks{Thanks to Spanish government (Ministerio de Educaci\'on y Ciencia, TEC2016-78434-C03-02) and European Union (FEDER) for funding.
}% <-this % stops a space
}

\hyphenation{op-tical net-works semi-conduc-tor hy-per-pa-ra-me-ters ge-ne-ra-tion ge-ne-ra-ted}

\maketitle

\begin{abstract}
%\notaRBT{no he tocado el abstract}
%Complex-valued signals are used in the modeling of many systems in engineering and science, hence being of fundamental interest.
 %Often, random complex-valued signals are considered to be proper. A proper complex random variable or process is uncorrelated with its complex conjugate. This assumption is a good model of the underlying physics in many problems, and simplifies the computations.
  %While linear processing and neural networks have been widely studied for complex-valued signals, the development of complex-valued non-linear kernel approaches remains an open problem. 
  
  Usually, complex-valued RKHS are presented as an straightforward application of the real-valued case. In this paper we prove that this procedure yields a limited solution for regression. We show that another kernel, here denoted as pseudo-kernel, is needed to learn any function in complex-valued fields. Accordingly, we derive a novel RKHS to include it, the widely RKHS (WRKHS). When the pseudo-kernel cancels, WRKHS reduces to complex-valued RKHS of previous approaches. We address the kernel and pseudo-kernel design, paying attention to the kernel and the pseudo-kernel being complex-valued. In the experiments included we report remarkable improvements in simple scenarios where real a imaginary parts have different similitude relations for given inputs or cases where real and imaginary parts are correlated. In the context of these novel results we revisit the problem of non-linear channel equalization, to show that the WRKHS helps to design more efficient solutions.

 \end{abstract}


% \fi
\IEEEpeerreviewmaketitle

\setlength{\arraycolsep}{0.1em} % para que no deje mucha separación en las ecuaciones

\section{Introduction}

\IEEEPARstart{C}{omplex-valued} signal processing is of fundamental interest. Its main benefit is the availability of processing the real and imaginary parts as a single signal. It finds application in a vast range of nowadays systems in science and engineering such as telecommunications, optics, electromagnetics, and acoustics among others. 
Signal processing for complex-valued signals has been widely studied in the linear case, see \cite{Schreier06} and references therein. 
%\tacha{On the other hand, until recently, no kernel-based methodology for treating complex signals had been developed.} 
The non-linear processing of complex-valued signals has been addressed from the point of view of neural networks, \cite{hirose13} and, recently, using reproducing kernel Hilbert spaces (RKHS). Some complex kernel-based algorithms have been lately proposed for classification \cite{Steinwart06}, regression \cite{OgunfunmiP11,Bouboulis11,Tobar12} and mainly for kernel principal component analysis \cite{Papaioannou14}. %and real 
%The kernel based solutions proposed for regression adopts real case structures. 
Regarding regression, in \cite{Bouboulis11} the authors propose a complex-valued kernel based in the results in \cite{Steinwart06} and face the derivative of cost functions by using Wirtinger's derivatives. Same kernel is adopted in \cite{OgunfunmiP11}. And in \cite{Bouboulis12} the augmented version of the algorithm is proposed.
%As discussed later in this paper the resulting approach involves properness. Besides, the kernel used is neither stationary nor isotropic, and it may suffer from numerical problems. 
In \cite{Tobar12} the authors review the kernel design to improve the previous solutions. 
%\tacha{They remarked that the kernel used in \cite{Bouboulis11} does not have the intuitive physical meaning of a measure of similarity of the samples and propose a kernel they denote as {\it independent}.} 
These previous approaches have been developed in the framework of kernel least mean square {(KLMS)}.
%The resulting kernel yields also proper complex-valued outputs. The kernel is stationary, but again it is not isotropic in the complex-valued input space, as the real and imaginary parts of the input are split and fed to different real valued kernels. Hence, these previous works do not report results for isotropic and stationary kernels that may better model the underlying physics of some systems. Also, the structure of the kernel remains more rigid than needed. These drawbacks make these solutions not powerful enough to learn a wide range of systems. 
%

Except for the method in \cite{Bouboulis12} all these algorithms are straight-forward applications of real-valued RKHS. In \cite{Bouboulis12} some additional considerations are developed for time adaptive estimation within the definition of the inner product in the feature space. These formulations, that are useful in the learning of many problems, are limited for learning in others. As we show in this paper they cannot learn any given complex-valued non-linear function. % as we could do with real and imaginary parts stacked in a vector. Also, the design of the kernel remains as an open question. In particular what the role of the imaginary part is. These issues are most important to be addressed to better face any complex-valued kernel regression problem.
%
%As a result, given a system to learn, model selection remains an important issue to be addressed. 
%
%We first must conclude if a pseudo-kernel is needed. Then, if the kernel, and the pseudo-kernel if needed, need to be complex-valued. 
%

%It is well known from estimation theory in complex-valued linear systems that straight forward transpositions of linear approaches yield limited solutions, denoted as \emph{strictly} linear approaches, and that a second term where conjugates are involved is needed. The solution included this term is know as 

In this paper we propose a novel RKHS for complex-valued fields with full representation capabilities. % to represent any given complex-valued function. 
We show that to represent any complex-valued function we need to include an additional term, denoted as pseudo-kernel\footnote{The pseudo-kernel plays a similar role of the pseudo-covariance of complex-valued random variables.}. We refer to this new approach as widely RKHS (WRKHS) after \emph{widely} linear complex valued solutions in linear systems \cite{Schreier06}. The results in  \cite{OgunfunmiP11,Bouboulis11} and \cite{Tobar12} can be seen as a particular case of WRKHS in which the pseudo-kernel is considered zero. We denote these approaches as strictly complex-valued RKHS (SRKHS). The need for the WRKHS can be justified in cases where the real and imaginary parts are correlated and learning them independently is, at best, suboptimal. Besides there are some relations that cannot even be capture with SRKHS approach, while our WRKHS, relaying on the pseudo-kernel, is able to learn on those scenarios, as we illustrate in the experimental section.   

%We provide the expressions for the kernel and the pseudo-kernel in terms of the kernel for the real part, the imaginary part and the real versus imaginary parts. As a first example of kernel, we propose a stationary isotropic Gaussian real-valued kernel of complex-valued inputs. Then, some proposal of kernel and pseudo-kernel design are developed. We analyze when the kernel and the pseudo-kernel are complex-valued. Particularly, in SRKHS we conclude that the imaginary part of the kernel is skew-symmetry or cancels.

One of the key issues with our WRKHS is the need to define kernels and pseudo-kernels. In this paper we describe valid kernels and pseudo-kernels. We also detail in which cases the WRKHS can be simplified to a SRKHS with complex or real-valued kernel.

Two experiments are included to illustrate the capabilities of WRKHS. First, we face a regression where clearly a different kernel for the real and imaginary parts benefits the learning. Then we learn a function using WRKHS with a real-valued kernel and a pure imaginary complex-valued pseudo-kernel. This solution allows modeling a dependence between real and imaginary part. Here, a WRKHS clearly improves the regression. We revisit the problem of non-channel equalization to conclude, from the results in this paper, that the best option is a SRKHS with real-valued kernel, even in the non-circular case. To compare to previous approaches we develop a version for the recursive case with sample selection \cite{Vaerenbergh12} and compare it to the results in \cite{Bouboulis11,Tobar12} for non-linear channel equalization. 

The paper is organized as follows. In the next section we review some concepts needed on RKHS. We continue in  \SEC{WRKHS} with the derivation of WRKHS. In \SEC{SRKHS} the SRKHS is developed as a special case of WRKHS. \SEC{Ker} is devoted to the analysis of the kernel and pseudo-kernel and some new proposals. %We develop in \SEC{ML} the optimization procedure to set the kernel hyperparameters applying Wirtinger's calculus and patterned complex-valued matrix derivatives. 
The performance of these approaches is illustrated in \SEC{Exp}, where several scenarios are presented and the application of WRKHS is discussed. We end with Conclusions.

The notation used in the paper is as follows. For matrix $\vect{A}$, $\entry{\vect{A}}{l}{q}$ is its $(l,q)$ entry, $\vect{A}\trs$ is the transpose of $\matr{A}$, $\vect{A}\her$ the Hermitian transpose, $\matr{A}\cnj$ its complex conjugate and $\matr{A}^{-*}$ its inverse conjugate. $\I_{{\n}}$ denotes the identity matrix of size $\n$. For a vector $\vect{a}$, $[\vect{a}]_l$ denotes its $l$-th entry. To denote the $i$-th sample of a vector and signal we use, respectively, $\vect{a}(i)$ and $a(i)$. The real and imaginary parts are denoted by subindex $\textrm{r}$ and $\textrm{j}$, respectively, i.e. $\vect{a}=\vect{a}\rr+\j \vect{a}\jj$, with $\textrm{j}=\sqrt{-1}$. To denote the complex Gaussian distribution with mean vector $\boldsymbol{\mu}$, covariance matrix $\K$ and pseudo-covariance matrix $\matr{\tilde{K}}$ we use $\calg{N}\left(\boldsymbol{\mu},\K,\matr{\tilde{K}}\right)$. We write the inner {or dot} product as $\langle\vect{a},\vect{b}\rangle=\vect{b}\her\vect{a}=\vect{a}\trs\vect{b}\cnj$.

%%%%%%%%%%%%%%%%%%%%%%%%%%%%%%%%%%%
\section{ RKHS} \LABSEC{RKHS}

%%Proponer un RKHS complejo suponiendo una phi compleja y alfas complejos.
A RKHS is a Hilbert space of functions that can be defined by a reproducing kernel $\k: \mathcal{X}\times \mathcal{X} \rightarrow \RN$  \cite{Scholkopf02}. Given the reproducing kernel $\k$, the RKHS $\mathcal{H}_\k$ of real-valued functions on the set $\mathcal{X}$ is the Hilbert space containing $\k(\x,\cdot)$ {for every $\x\in\mathcal{X}$ and where $\k$ has the \emph{reproducing property}} 
\begin{equation}
\begin{matrix}
\f(\x\new)=\langle \f,\k(\x\new,\cdot)\rangle_\k & \;\;\;\forall \f\in \mathcal{H}_\k\end{matrix},
\end{equation}
 being $\langle\cdot,\cdot\rangle_\k$ the inner product in $\mathcal{H}_\k$. {In particular, $\langle \k(\x,\cdot),\k(\x\new,\cdot)\rangle_\k=\k(\x,\x\new)$.}
In a RKHS, functions are in the closure of the linear combinations of the kernel at given points:
\begin{equation}\LABEQ{rkhs}
\f(\x\new)=\sum_{i=1}^{\n}\alpha_i\k(\x\new,\x({i}))=\kv(\x\new,\X) \alfav,
\end{equation}
%\notaRBT{PREGUNTA IMPORTANTE: ES $\k(\x\new,\x({i}))$ O $\k(\x({i}),\x\new)$?}

where $\f$ is in the class $\mathcal{F}$ of real functions forming a real Hilbert space, $\alfav=[\alpha_1,\alpha_2,...,\alpha_\n]\trs$, and {$\kv(\x\new,\X)=[\k(\x\new,\x({1})), \k(\x\new,\x({2})),...,\k(\x\new,\x({\n}))]$}.

In the complex-valued case, one might work with the real and imaginary parts stacked into a so denoted \emph{composite} vector form. The definition of RKHS for vector valued functions parallels the one in the scalar, with the main difference that the reproducing kernel is now matrix valued \cite{Alvarez12},
\begin{equation}\LABEQ{MOL}
\fv\com{(\x\new)}=\begin{bmatrix}\f\rr(\x\new)\\ \f\jj(\x\new) \end{bmatrix}= 
\begin{bmatrix}
%\phi\rr(\x\new)\trs\Phi\rr(\X)+\phi\jj(\x\new)\trs\Phi\jj(\X)  
\kv\rrrr(\x\new,\X)
&
\kv\rrjj(\x\new,\X)
%\phi\rr(\x\new)\trs\Phi\jj(\X)-\phi\jj(\x\new)\trs\Phi\rr(\X) 
\\
\kv\jjrr(\x\new,\X)
%-\phi\rr(\x\new)\trs\Phi\jj(\X)+\phi\jj(\x\new)\trs\Phi\rr(\X) 
&
\kv\jjjj(\x\new,\X)
%\phi\rr(\x\new)\trs\Phi\rr(\X)+\phi\jj(\x\new)\trs\Phi\jj(\X)  
 \end{bmatrix}\begin{bmatrix}\alfav\rr\\ \alfav\jj
  \end{bmatrix}%=\kv\com(\x\new,\X)\alfav\com.%,
 \end{equation}
that can be rewritten in compact form as $\fv\com{(\x\new)}={\K}\com(\x\new,\X)\alfav\com$.
We have two-dimensional vector, {and we can define an estimator by minimizing the regularized empirical error on the basis of a training set $\tset=\{\X,\yv	\}=\{(\x(1),\y(1)), . . . ,(\x(\n),\y(\n))\}$}:
 \begin{equation}
\frac{1}{\n}\sum_{i=1}^{\n}\left( \f\rr(\x(i) ) - {\y\rr}{(i)} \right)^2+\frac{1}{\n}\sum_{i=1}^{\n}\left( \f\jj(\x(i) ) - {\y\jj}{(i)} \right)^2+\frac{\lambda}{n}\|\fv\com \|_{\K}^2
\end{equation}
 the coefficients yield
\begin{equation} \LABEQ{compositealfa}
\alfav\com=\left( \K\com(\X,\X) +\lambda\I_{{2\n}} \right)\inv\yv\com,
\end{equation}
where $\yv\com=[\yv\rr\trs \yv\jj\trs]\trs$, with $\yv\rr=[\y\rr(1), . . . ,\y\rr(\n)]\trs$ and $\yv\jj=[\y\jj(1), . . . ,\y\jj(\n)]\trs$.
%\notaRBT{ PREGUNTA: DEBEMOS DEFINIR $\K\com(\X,\X)$? }

%%%%%%%%%%%%%%%%%%%%%%%%%%%%%%%%%%%%%%
%%%%%%%%%%%%%%%%%%%%%%%%%%%%%%%%%%%%%%

%%%%%%%%%%%%%%%%%%%%%%%%%%%%%%%%%%%%%%
%%%%%%%%%%%%%%%%%%%%%%%%%%%%%%%%%%%%%%
\section{Widely Complex RKHS} \LABSEC{WRKHS}

Based on the \emph{widely linear} concept \cite{Schreier06} we propose the following RKHS for regression in complex-valued formulation.
\begin{definition}\emph{Widely complex RKHS}. We denote as widely complex-valued RKHS (WRKHS) the RKHS defined by the kernel $\k: \mathcal{X}\times \mathcal{X} \rightarrow \CN$ {and a \emph{pseudo-kernel} $\pk: \mathcal{X}\times \mathcal{X} \rightarrow \CN$}, 
\begin{align}\LABEQ{WRKHS}
\f(\x\new)&=\sum_{i=1}^{\n}\alpha_i\k(\x\new,\x({i}))+\sum_{i=1}^{\n}\alpha_i\cnj\pk(\x\new,\x({i})) 
%\\
%&=\sum_{i=1}^{\n}\alpha_i \langle \pv(\x\new),\pv(\x({i}))\rangle + \sum_{i=1}^{\n}\alpha_i\cnj \langle \pv(\x\new),\pv(\x({i}))\cnj\rangle \nonumber
\end{align}
where $\alpha_i \in \CN$. {}
\end{definition}
The \emph{pseudo-kernel} is related to the feature map,  $\pv : \mathcal{X}\rightarrow \CN${$^{\q}$, by $\pk(\x\new,\x)=\langle \pv(\x\new),\pv(\x)\cnj\rangle$. We introduce the following definitions that we need in the next proposition.
%Let us first introduce the following definitions. 
\begin{definition}\emph{Kernels of real-imaginary parts of the feature space}. We define the kernels for the real to real, real to imaginary, imaginary to real and imaginary to imaginary parts of the feature space, respectively, as
{\begin{equation}\LABEQ{defgammas}
\begin{matrix}
\krj\rrrr(\x\new,\x)=\langle\pv\rr(\x\new),\pv\rr(\x)\rangle,
\\
\krj\rrjj(\x\new,\x)=\langle\pv\rr(\x\new),\pv\jj(\x)\rangle,
\\
\krj\jjrr(\x\new,\x)=\langle\pv\jj(\x\new),\pv\rr(\x)\rangle,
\\
\krj\jjjj(\x\new,\x)=\langle\pv\jj(\x\new),\pv\jj(\x)\rangle,
 \end{matrix}
\end{equation}
where $\krj\rrjj(\x\new,\x)=\krj\jjrr(\x,\x\new)$}. 
%\begin{equation}\LABEQ{defgammas}
%\begin{matrix}
%\krjv\rrrr(\x\new,\x)=\langle\pv\rr(\x\new),\PHIv\rr(\x)\rangle,
%\\
%\krjv\rrjj(\x\new,\x)=\langle\pv\rr(\x\new),\PHIv\jj(\x)\rangle,
%\\
%\krjv\jjrr(\x\new,\x)=\langle\pv\jj(\x\new),\PHIv\rr(\x)\rangle,
%\\
%\krjv\jjjj(\x\new,\x)=\langle\pv\jj(\x\new),\PHIv\jj(\x)\rangle,
% \end{matrix}
%\end{equation}
%{where $\Phi(\X)$ is an $m \times \n$ matrix  whose $i$-th column is $\pv(\x({i}))${, and $\kv(\x\new,%\X)=[\k(\x\new,\x({1})),\k(\x\new,\x({2})),...,\k(\x\new,\x({\n}))]$}}
\end{definition}

 \begin{proposition}\emph{WRKHS reproducing properties.} \LABPRP{WRKHS}
The WRKHS can learn the real and the imaginary parts of the output as in \EQ{MOL}. 
\end{proposition}
\begin{proof}
The output \EQ{WRKHS} as a function of the feature space can be rewritten as
\begin{align}%\LABEQ{WRKHS}
\f(\x\new)%&=\sum_{i=1}^{\n}\alpha_i\k(\x\new,\x({i}))+\sum_{i=1}^{\n}\alpha_i\cnj\k(\x\new,\x({i})\cnj) 
%\\
&=\sum_{i=1}^{\n}\alpha_i \langle \pv(\x\new),\pv(\x({i}))\rangle + \sum_{i=1}^{\n}\alpha_i\cnj \langle \pv(\x\new),\pv(\x({i}))\cnj\rangle \nonumber.
\end{align}
In {composite} form it follows that
\begin{align}\LABEQ{WEF}
\begin{bmatrix}f\rr(\x\new)\\f\jj(\x\new) \end{bmatrix}&= 
 \begin{bmatrix}
{{\pv\rr}\new}\trs
&
-{{\pv\jj}\new}\trs
\\
{{\pv\jj}\new}\trs
&
{{\pv\rr}\new}\trs
 \end{bmatrix}
 \!\!\left(
 \begin{bmatrix}
\PHIv\rr 
&
\PHIv\jj
\\
-\PHIv\jj
&
\PHIv\rr
 \end{bmatrix}
 \!\!+\!\!
 \begin{bmatrix}
\PHIv\rr 
&
-\PHIv\jj
\\
\PHIv\jj
&
\PHIv\rr
 \end{bmatrix}
 \right)\!\!
 \begin{bmatrix}\alfav\rr\\ \alfav\jj \end{bmatrix}
 \nonumber
 \\
&= 
 2\begin{bmatrix}
{{\pv\rr}\new}\trs\PHIv\rr 
&\;\;\;&
{{\pv\rr}\new}\trs\PHIv\jj
\\
{{\pv\jj}\new}\trs\PHIv\rr 
&\;\;\;&
{{\pv\jj}\new}\trs\PHIv\jj 
 \end{bmatrix}
 \begin{bmatrix}\alfav\rr\\ \alfav\jj \end{bmatrix}
 \nonumber
 \\
&= 
 2\begin{bmatrix}
\krjv\rrrr(\x\new,\X)
&\;\;&
\krjv\rrjj(\x\new,\X) 
\\
\krjv\jjrr(\x\new,\X) 
&\;\;&
\krjv\jjjj(\x\new,\X) 
 \end{bmatrix}
 \begin{bmatrix}\alfav\rr\\ \alfav\jj \end{bmatrix}.
 \end{align}
%where ${\pv\rr}\new=\pv\rr(\x\new)$, ${\phi\jj}\new=\phi\jj(\x\new)$, $\pv\rr=\pv\rr(\X)$, $\pv\jj=\pv\jj(\X)$ and $\pv(\X)$ is an $m \times \n$ matrix  whose $i$-th column is $\pv(\x({i}))$. %{, and $\kv(\x\new,\X)=[\k(\x\new,\x({1})),\k(\x\new,\x({2})),...,\k(\x\new,\x({\n}))]$}. 
This corresponds to the approach in \EQ{MOL} 
where the entries of the matrix in ${\K}\com(\x\new,\X)$ can be easily identified. $\blacksquare$
\end{proof}

In the proof of \PRP{WRKHS} we added both terms in \EQ{WRKHS}. 
An interesting conclusion can be drawn if we develop each term independently, 
\begin{align}\LABEQ{WEF2}
&\begin{bmatrix}f\rr(\x\new)\\f\jj(\x\new) \end{bmatrix}= 
\begin{bmatrix}
{{\pv\rr}\new}\trs\PHIv\rr+{{\pv\jj}\new}\trs\PHIv\jj  
%\kv_1(\x\new,\X)
&\; \;&
%-\kv_2(\x\new,\X)
{{\pv\rr}\new}\trs\PHIv\jj-{{\pv\jj}\new}\trs\PHIv\rr 
\\
%\kv_2(\x\new,\X)
-{{\pv\rr}\new}\trs\PHIv\jj+{{\pv\jj}\new}\trs\PHIv\rr 
&\; \;&
%\kv_1(\x\new,\X)
{{\pv\rr}\new}\trs\PHIv\rr+{{\pv\jj}\new}\trs\PHIv\jj  
 \end{bmatrix}
 \begin{bmatrix}\alfav\rr\\ \alfav\jj \end{bmatrix}
\nonumber
\\
&+
\begin{bmatrix}
{{\pv\rr}\new}\trs\PHIv\rr-{{\pv\jj}\new}\trs\PHIv\jj  
%\kv_1(\x\new,\X)
&\; \;&
%-\kv_2(\x\new,\X)
-{{\pv\rr}\new}\trs\PHIv\jj-{{\pv\jj}\new}\trs\PHIv\rr 
\\
%\kv_2(\x\new,\X)
+{{\pv\rr}\new}\trs\PHIv\jj+{{\pv\jj}\new}\trs\PHIv\rr 
&\; \;&
%\kv_1(\x\new,\X)
-{{\pv\rr}\new}\trs\PHIv\rr+{{\pv\jj}\new}\trs\PHIv\jj  
 \end{bmatrix}
 \begin{bmatrix}\alfav\rr\\ -\alfav\jj \end{bmatrix}.
 \end{align}
The two matrices above resemble the covariance and the pseudo-covariance in complex-valued Gaussian distributions, respectively. %\notaRBT{PREGUNTA: TIENEN QUE SER GAUSSIAN O VALE PARA CUALQUIER RANDOM?}. 
%The first matrix, resembling the covariance. 
% The second one plays a role similar to that of the pseudo-covariance. It
The pseudo-covariance cancels for conditions similar to those of the proper case in complex-valued random variables: if ${{\pv\rr}\new}\trs\pv\rr={{\pv\jj}\new}\trs\pv\jj$ and ${{\pv\rr}\new}\trs\pv\jj=-{{\pv\jj}\new}\trs\pv\rr$. From the result in \EQ{WEF2}, the WRKHS can be rewritten in complex-valued form as follows
\begin{equation}\LABEQ{fvw}
\f(\x\new)=\kv(\x\new,\X)\alfav+\tilde{\kv}(\x\new,\X)\alfav\cnj,
\end{equation}
%Next step is to derive the value for the complex $alpha$. 
%
%\section{Widely Complex Gaussian Process Regression} \LABSEC{aug}
%%\subsection{Augmented Complex GPR}
%
%%In this section we derive the GPR for the augmented form, as a previous step to the computation of the complex version of the GPR. 
%The model in \EQ{regression} can also be rewritten in the augmented vector notation by stacking the complex signals on their conjugates:
%\begin{align}
%\aug{\yv}=\left[ \begin{array}{c}
%\yv\\
%\yv^{*}\\
%\end{array}\right]=\left[ \begin{array}{c}
%{\fv(\X_\n)}\\
%{\fv}^{*}(\X_\n)\\
%\end{array}\right]+\left[ \begin{array}{c}
%{\noiseOutv}\\
%{\noiseOutv}^{*}\\
%\end{array}\right] &=\aug{\fv}(\X_\n)+\aug{\noiseOutv}\nonumber\\&=\aug{\fv}+\aug{\noiseOutv},\label{eq:modelaugcomplexlinear}
%\end{align}
%where ${\aug{\yv}}$, $\aug{\fv}$ and $\aug{\noiseOutv}$ are the augmented vectors for the outputs, the latent function vector and the noise, respectively. 
{and} we next derive the value for $\alfav$ {by minimizing
the regularized empirical error. We will make use of the \emph{augmented} vector $\aug{\fv}(\x\newd)=[{\f}(\x\newd) \;{\f}\cnj(\x\newd)]\trs$:
%NOTA: HE CAMBIADO LA FORMA DE CALCULAR LO QUE SIGUE, PORQUE NO SABIA COMO EXPLICAR DE DONDE SALE EL RUIDO Y C SIN ENTRAR EN GPR
\begin{equation}\LABEQ{augfv}
\aug{\fv}(\x\newd)
%=\left[ \begin{array}{c}
%{\f}(\x\newd)\\
%{\f}\cnj(\x\newd)\\
%\end{array}\right]
=\begin{bmatrix}
\kv(\x\new,\X) & \tilde{\kv}(\x\new,\X)\\
\tilde{\kv}\cnj(\x\new,\X) & \kv\cnj(\x\new,\X) 
\end{bmatrix}
\begin{bmatrix}\alfav\\ \alfav\cnj \end{bmatrix}
 =\aug\K(\x\new,\X)\aug\alfav,
 \end{equation}
where matrix $\aug{\K}(\x\new,\X)$ is the \emph{augmented kernel matrix}, and $\aug\alfav=[\alfav\trs \;\alfav\her]\trs$ is also an augmented vector.}
There exists a simple relation between the composite and the \emph{augmented} vector, {$\aug{\fv}(\x\newd)=\T\fv\com(\x\new)$}, where  
\begin{equation}
\T=\left[ \begin{array}{c c}
\I_{\n} & \j\I_{\n} \\
\I_{\n}  & -\j\I_{\n} \\
\end{array}\right]\in \mathbb{C}^{2n\times 2n},
\end{equation}
and $\T\T\her=\T\her\T=2\I_{2\n}$.
%\tacha{Also,  $\aug{\noiseOutv}=\T\noiseOutv\com$ and $\aug{\fv}=\T\fv\com$. }
%The linear regression model in (\ref{eq:modelcomplexlinear}) gives an example of a complex gaussian process. Given a training set $(\X, \yv)$, the goal is to predict the value of $\fv\newd\triangleq \fv(\matr{X\newd})$ for new inputs $\X\newd$. The starting point is the conditional distribution of the prediction ${{\fv\com}\newd}| \X\newd,\X,\yv\com \sim \calg{N}\left(\vect{\mu}_{{\fv\com}\newd},\matr{\Sigma}_{{\fv\com}\newd}\right)$ found in the previous section, (\ref{eq:meanreal}) and (\ref{eq:varreal}). 
This simple transformation allows us to calculate the augmented vector {\EQ{augfv} from the real and imaginary parts \EQ{MOL}:
\begin{align}
\aug{\fv}(\x\newd)&=\T\K\com(\x\new,\X)\alfav\com=\T\K\com(\x\new,\X)\frac{1}{2}\T\her\T\alfav\com
\nonumber\\
&=\frac{1}{2}\T\K\com(\x\new,\X)\T\her\aug{\alfav}.
\end{align}
Hence, the augmented kernel matrix is related to $\K\com(\x\new,\X)$ as 
\begin{equation}\LABEQ{relaugKcompK}
\aug{\K}(\x\new,\X)=\frac{1}{2}\T\K\com(\x\new,\X)\T\her.
\end{equation}
And the augmented vector $\aug{\alfav}$ can be found from \EQ{compositealfa} as
\begin{align}\LABEQ{alfavwa}
\aug{\alfav}&=\T\alfav\com=\T\left( \K\com(\X,\X) +\lambda\I_{2\n} \right)\inv\yv\com\nonumber\\
&=\left(\left( \K\com(\X,\X) +\lambda\I_{2\n}\right) \T\inv\right)\inv\frac{1}{2}\T\her\T\yv\com\nonumber\\
&=\frac{1}{2}\left((\T\her)\inv\left( \K\com(\X,\X) +\lambda\I_{2\n}\right) \T\inv\right)\inv\aug{\yv}\nonumber\\
&=\frac{1}{2}\left(\frac{1}{2}\T\left( \K\com(\X,\X) +\lambda\I_{2\n}\right)\frac{1}{2}\T\her\right)\inv\aug{\yv}\nonumber\\
&=\left(\aug{\K}(\X,\X)+\lambda\I_{2\n}\right)\inv\aug{\yv}.
\end{align}
%Note that in \EQ{augfv} the \emph{augmented kernel matrix} has a particular block pattern, with the kernel in the diagonal blocks and the pseudo-kernel in the off-diagonal ones, and with the second row conjugated. 
}
%
%%%%%%Comento lo siguiente. Rafa
%of the estimation which is $\aug{\fv}(\x\newd)=\T{\fv\com}(\x\newd)$ :
%\begin{equation}\LABEQ{fvwa}
%\aug{\fv}(\x\newd)=\left[ \begin{array}{c}
%{\fv}(\x\newd)\\
%{\fv}\cnj(\x\newd)\\
%\end{array}\right]=\aug{\K}(\X\newd,\X_\n)\aug{\matr{C}}\inv\aug{\yv},
%\end{equation}
%where the $\aug{\matr{C}}$ is defined as
%\begin{align} \label{eq:augC}
%\aug{\matr{C}}&=\left[ \begin{array}{cc}
%\matr{C} & \matr{\tilde{C}}\\
%\matr{\tilde{C}}\cnj&\matr{C}^{*}\\
%\end{array}\right]=\aug{\K}(\X_\n,\X_\n)+\aug{\matr{\Sigma}}_{\noiseOutv}\nonumber\\&=\left[ \begin{array}{cc}
%\K(\X_\n,\X_\n)& \matr{\tilde{K}}(\X_\n,\X_\n)\\
%\matr{\tilde{K}}^{*}(\X_\n,\X_\n)^* &\K^{*}(\X_\n,\X_\n)\\
%\end{array}\right]
%+\left[ \begin{array}{cc}
%\matr{\Sigma} & \matr{\tilde\Sigma}\\
%\matr{\tilde\Sigma}\cnj &\matr{\Sigma}\cnj\\
%\end{array}\right]
%\end{align}
%$\matr{\Sigma}=\matr{\Sigma}\cnj=\lambda\I$ and $\matr{\tilde\Sigma}=\matr{\tilde\Sigma}\cnj=0$. %Matrix $\aug{\matr{\Sigma}}=\T\Sigma\T\her$ is the augmented covariance matrix of the regularization term, and $\aug{\K}(\X_\n,\X_\n)=\T\K\com(\X_\n,\X_\n)\T\her$ is the augmented covariance matrix of $\aug{\fv}=\aug{\fv}(\X_\n)$, composed by the covariance matrix $\K(\X_\n,\X_\n)= \mathbb{E}\left[\fv(\X_\n)\fv\her(\X_\n)\right]$ and the pseudo-covariance or complementary covariance matrix $\matr{\tilde{K}}(\X_\n,\X_\n)=\mathbb{E}\left[\fv(\X_\n)\fv^\top(\X_\n)\right]$.
%%%%%%% Hasta aquí comentado. Rafa
%
Note that in the general complex case, two functions {$\k(\x\new,\x)$ and $\pk(\x\new,\x)$} must be defined. {By identifying $\K\com(\x\new,\X)$ in \EQ{WEF} and substituting it in \EQ{relaugKcompK}, 
\begin{equation}
\aug{\K}(\x\new,\X)=\frac{1}{2}\T\left(2\begin{bmatrix}
\krjv\rrrr(\x\new,\X)
&\;\;&
\krjv\rrjj(\x\new,\X) 
\\
\krjv\jjrr(\x\new,\X) 
&\;\;&
\krjv\jjjj(\x\new,\X) 
 \end{bmatrix}\right)\T\her.
\end{equation}
the kernel and pseudo-kernel can be identified:
\begin{align}
\k(\x\new,\x({i}))&=\krj\rrrr(\x\new,\x({i}))+\krj\jjjj(\x\new,\x({i}))\nonumber\\&+\j\left(\krj\jjrr(\x\new,\x({i}))-\krj\rrjj(\x\new,\x({i}))\right),
%{\K}(\X_\n,\X_\n)&=\K\rrrr(\X_\n,\X_\n)+ \K\jjjj(\X_\n,\X_\n)\nonumber\\&+\j\left(\K\rrjj\trs(\X_\n,\X_\n)-\K\rrjj(\X_\n,\X_\n)\right),
\LABEQ{covK}\\
%{\matr{\tilde{K}}}(\X,\X)&=\K\rrrr(\X_\n,\X_\n)- \K\jjjj(\X_\n,\X_\n)\nonumber\\&+\j\left(\K\rrjj\trs(\X_\n,\X_\n)%+\K\rrjj(\X_\n,\X_\n) \right).
\pk(\x\new,\x({i}))&=\krj\rrrr(\x\new,\x({i}))-\krj\jjjj(\x\new,\x({i}))\nonumber\\&+\j\left(\krj\jjrr(\x\new,\x({i}))+\krj\rrjj(\x\new,\x({i}))\right).\LABEQ{pcovK}
\end{align}}
{Finally, by applying the matrix-inversion lemma to \EQ{alfavwa},
\begin{align}
\left(\aug{\K}(\X,\X)+\lambda\I_{2\n}\right)\inv&=\begin{bmatrix}
\K(\X,\X)+\lambda\I_{\n} & \pK(\X,\X)\\
\pK\cnj(\X,\X) &\K\cnj(\X,\X)+\lambda\I_{\n}\end{bmatrix}\inv
\end{align}
and substitution in \EQ{augfv}, it yields the prediction
\begin{align}\LABEQ{fWRKHS}
\f(\x\new)&=\kv(\x\new,\X)\left[\matr{P}\inv\yv- \C\inv\pK(\X,\X)\matr{P}^{-*}\yv\cnj \right]\nonumber\\
&+\pkv(\x\new,\X)\left[ \matr{P}^{-*}\yv\cnj- \C^{-*}\pK\cnj(\X,\X)\matr{P}\inv\yv\right]
\end{align}
%\begin{align}\LABEQ{fWRKHS}
%{\fv}\newd&=\K(\X\newd,\X_\n)\left[  \matr{P}\inv\yv- \matr{{C}}\inv\matr{\tilde{C}}\matr{P}^{-*}\yv^{*} \right]+\nonumber\\
%&
% \matr{\tilde{K}}(\X\newd,\X_\n)\left[ \matr{P}^{-*}\yv^{*}- \matr{C}^{-*}\matr{\tilde{C}}^{*}\matr{P}\inv\yv\right]
%%\\&
%%=\left[\K(\X\newd,\X_\n)- \matr{\tilde{K}}(\X\newd,\X_\n)\matr{C}^{-*}\matr{\tilde{C}}^{*}\right]\matr{P}\inv\yv\nonumber\\&+\left[ \matr{\tilde{K}}(\X\newd,\X_\n)-\K(\X\newd,\X_\n)\matr{{C}}\inv\matr{\tilde{C}}\right]\matr{P}^{-*}\yv^{*},
%\end{align}
where $\C=\left(\K(\X,\X)+\lambda\I_{\n}\right)$, and $\matr{P}=\C-\pK(\X,\X)\C^{-*}\pK\cnj(\X,\X)$. Now the pair $\alfav$, $\alfav\cnj$ are easily identifiable. }

%%%%%%%%%%%%%%%%%%%%%%%%%%%%%%%%%%%%%%%
%%%%%%%%%%%%%%%%%%%%%%%%%%%%%%%%%%%%%%%
 \section{ Strictly Complex RKHS} \LABSEC{SRKHS}
 %\subsection{SRKHS as special case of the WRKHS}
 By removing the last term in \EQ{WRKHS} we have a particular case of the WRKHS that we denote as strictly complex-valued RKHS.
\begin{definition}\emph{Strictly complex-valued RKHS}. We denote as {strictly} complex-valued RKHS (SRKHS) the RKHS defined by the kernel $\k: \mathcal{X}\times \mathcal{X} \rightarrow \CN$, 
\begin{equation}\LABEQ{SRKHS}
\f(\x\new)=\sum_{i=1}^{\n}\alpha_i\k(\x\new,\x({i}))=\sum_{i=1}^{\n}\alpha_i \langle \pv(\x\new),\pv(\x({i}))\rangle  
\end{equation}
where $\alpha_i \in \CN$ and the feature map is given by $\pv : \mathcal{X}\rightarrow \CN${$^{\q}$}.
%\notaRBT{EL FEATURE MAP ES UN VECTOR, NO?, DE QUE DIMENSIONES?}. 
\end{definition}
%
%and . %, $\x({i}) \in \CN^d$.
%In RKHS the kernel can be written as the inner product of the inputs in the feature space given by  ,

It can be proved to be a RKHS by using complex-valued Hilbert spaces, see \cite{Paulsen09}. %In the following, we develop its properties. 
This RKHS, that it is a straightforward application of the real-valued RKHS, is limited compared to the WRKHS as we show next. 
\begin{proposition}\emph{SRKHS is limited as RKHS.} \LABPRP{SRKHS}
The SRKHS is limited to represent any given complex-valued function. In particular, it yields a subset of the functions that WRKHS can represent. %Furthermore, if real and imaginary parts of the outputs are independent, the SRKHS yields de complexification approach.
\end{proposition}
\begin{proof}
By rewriting the output {\EQ{SRKHS} in composite}  form, i.e. real and imaginary parts stacked in vector form, 
\begin{align}\LABEQ{PRJ}
\begin{bmatrix}f\rr(\x\new)\\f\jj(\x\new) \end{bmatrix}= 
%%\begin{bmatrix}
%%\phi\rr(\x\new)\trs\Phi\rr(\X)+\phi\jj(\x\new)\trs\Phi\jj(\X)  
%%%\kv_1(\x\new,\X)
%%&
%%%-\kv_2(\x\new,\X)
%%\phi\rr(\x\new)\trs\Phi\jj(\X)-\phi\jj(\x\new)\trs\Phi\rr(\X) 
%%\\
%%%\kv_2(\x\new,\X)
%%-\phi\rr(\x\new)\trs\Phi\jj(\X)+\phi\jj(\x\new)\trs\Phi\rr(\X) 
%%&
%%%\kv_1(\x\new,\X)
%%\phi\rr(\x\new)\trs\Phi\rr(\X)+\phi\jj(\x\new)\trs\Phi\jj(\X)  
%% \end{bmatrix}
% \begin{bmatrix}
%{{\pv\rr}\new}\trs\pv\rr+{{\pv\jj}\new}\trs\pv\jj  
%%\kv_1(\x\new,\X)
%&\;\;&
%%-\kv_2(\x\new,\X)
%{{\pv\rr}\new}\trs\pv\jj-{{\pv\jj}\new}\trs\pv\rr 
%\\
%%\kv_2(\x\new,\X)
%-{{\pv\rr}\new}\trs\pv\jj+{{\pv\jj}\new}\trs\pv\rr 
%&\;\;&
%%\kv_1(\x\new,\X)
%{{\pv\rr}\new}\trs\pv\rr+{{\pv\jj}\new}\trs\pv\jj  
% \end{bmatrix}
% \begin{bmatrix}\alfav\rr\\ \alfav\jj \end{bmatrix}
{\begin{bmatrix}
{{\pv\rr}\new}\trs\PHIv\rr+{{\pv\jj}\new}\trs\PHIv\jj  
%\kv_1(\x\new,\X)
&\; \;&
%-\kv_2(\x\new,\X)
{{\pv\rr}\new}\trs\PHIv\jj-{{\pv\jj}\new}\trs\PHIv\rr 
\\
%\kv_2(\x\new,\X)
-{{\pv\rr}\new}\trs\PHIv\jj+{{\pv\jj}\new}\trs\PHIv\rr 
&\; \;&
%\kv_1(\x\new,\X)
{{\pv\rr}\new}\trs\PHIv\rr+{{\pv\jj}\new}\trs\PHIv\jj  
 \end{bmatrix} }
 \begin{bmatrix}\alfav\rr\\ \alfav\jj \end{bmatrix}
 \end{align}
where ${{\pv\rr}\new}=\pv\rr(\x\new)$, ${\pv\jj}\new=\pv\jj(\x\new)$, {$\PHIv\rr=\PHIv\rr(\X)$}, {$\PHIv\jj=\PHIv\jj(\X)$} and {$\PHIv(\X)$} is an $m \times \n$ matrix  whose $i$-th column is $\pv(\x({i}))$. %{, and $\kv(\x\new,\X)=[\k(\x\new,\x({1})),\k(\x\new,\x({2})),...,\k(\x\new,\x({\n}))]$}. 
It can be observed that the diagonal {blocks} of the matrix above have the same value {while the off-diagonal ones have opposite sign}. Hence, it cannot provide the same solutions than the WKRHS in \EQ{WRKHS} where in the general case $\kv\rrrr(\x\new,\X)\neq\kv\jjjj(\x\new,\X)$ and $\kv\rrjj(\x\new,\X)\neq-\kv\jjrr(\x\new,\X)$. %Besides, if real and imaginary parts are independent, off-diagonal terms vanish, and we have the complexification approach. 
$\blacksquare$
\end{proof}%$\blacksquare$

The previous proposition is a consequence of the fact that linear operations applied to a real vector formed with the real and imaginary parts are not generally translated into linear operations applied to its complex counterpart.

\subsection{Kernel structure}
We next study the structure of the kernel {for} the SRKHS. %
\begin{proposition}\emph{Kernel in SRKHS.} \LABPRP{SRKHS2}
The solution in \EQ{MOL} with $k\rrrr(\x\new,\x({i}))=k\jjjj(\x\new,\x({i}))=k\rrrr(\x({i}),\x\new)$ and $\k\rrjj(\x\new,\x({i}))=-\k\rrjj(\x({i}),\x\new)$ yields the SRKHS in \EQ{SRKHS} with kernel 
\begin{equation}\LABEQ{PK}
k(\x\new,\x({i}))=k\rrrr(\x\new,\x({i}))-\j\k\rrjj(\x\new,\x({i})).
\end{equation}
%\notaRBT{where $k\rrrr(\x\new,\x({i}))=k\rrrr(\x({i}),\x\new)$, while $\k\rrjj(\x\new,\x({i}))=-\k\rrjj(\x({i}),\x\new)$.}
\end{proposition}
\begin{proof}
First, we rewrite \EQ{SRKHS} in vector form as
\begin{equation}\LABEQ{fvs}
f(\x\new)=\pv(\x\new)\trs\Phi(\X)\cnj\alfav=\kv(\x\new,\X)\alfav
\end{equation}
and decompose it into real and imaginary parts as in \EQ{PRJ}.
Define
\begin{equation}\LABEQ{k12}
\begin{matrix}
%\kv_1(\x\new,\X)= \phi\rr(\x\new)\trs\Phi\rr(\X)+\phi\jj(\x\new)\trs\Phi\jj(\X)  
\kv\rrrr(\x\new,\X)= \krjv\rrrr(\x\new,\X)+\krjv\jjjj(\x\new,\X),
\\
%\kv_2(\x\new,\X)=\phi\jj(\x\new)\trs\Phi\rr(\X) - \phi\rr(\x\new)\trs\Phi\jj(\X)
\kv\rrjj(\x\new,\X)=\krjv\rrjj(\x\new,\X)- \krjv\jjrr(\x\new,\X),
 \end{matrix}
  \end{equation}
{where $\krjv\rrrr(\x\new,\X)=[\krj\rrrr(\x\new,\x({1})),\cdots,\krj\rrrr(\x\new,\x({\n}))]$, and we have analogous definitions for $\krjv\jjjj(\x\new,\X)$, $\krjv\rrjj(\x\new,\X)$ and $\krjv\jjrr(\x\new,\X)$. The terms in \EQ{k12} can be identified in \EQ{PRJ} as}
\begin{equation}\LABEQ{k12b}
\begin{bmatrix}f\rr(\x\new)\\f\jj(\x\new) \end{bmatrix}= 
\begin{bmatrix}
%\phi\rr(\x\new)\trs\Phi\rr(\X)+\phi\jj(\x\new)\trs\Phi\jj(\X)  
\kv\rrrr(\x\new,\X)
&
\kv\rrjj(\x\new,\X)
%\phi\rr(\x\new)\trs\Phi\jj(\X)-\phi\jj(\x\new)\trs\Phi\rr(\X) 
\\
-\kv\rrjj(\x\new,\X)
%-\phi\rr(\x\new)\trs\Phi\jj(\X)+\phi\jj(\x\new)\trs\Phi\rr(\X) 
&
\kv\rrrr(\x\new,\X)
%\phi\rr(\x\new)\trs\Phi\rr(\X)+\phi\jj(\x\new)\trs\Phi\jj(\X)  
 \end{bmatrix}\begin{bmatrix}\alfav\rr\\ \alfav\jj \end{bmatrix}
 \end{equation}
where it can be concluded that $k\rrrr(\x\new,\x({i}))=k\jjjj(\x\new,\x({i}))$.
%We next show that this expression equals that of a vector or multiple output system, hence being a valid RHKS for vector-valued functions. 
% Decir que es un RKHS y que puede verse descomponiendo parte real e imaginaria
{Going back to \EQ{SRKHS} it follows that $\k(\x\new,\x({i}))=\k\rrrr(\x\new,\x({i}))-\j\k\rrjj(\x\new,\x({i}))$. 
Finally, from the definitions of $k\rrrr(\x\new,\x({i}))$ and $\kv\rrjj(\x\new,\X)$ in \EQ{k12} and definitions in \EQ{defgammas}, it is easy to check the symmetries $k\rrrr(\x\new,\x({i}))=k\rrrr(\x({i}),\x\new)$ and $\k\rrjj(\x\new,\x({i}))=-\k\rrjj(\x({i}),\x\new)$.} 
\end{proof}$\blacksquare$

%It can be concluded that \EQ{SRKHS} yields a RHKS for vector-valued functions where real and imaginary parts present the same kernel, $k\rr=k\jj=k_1$, and the cross-kernel are given by $k\rrjj=-k\jjrr=-k_2$. Furthermore, 
Note first that by minimizing
the regularized empirical error $\alfav$ in \EQ{fvs} we have 
\begin{equation}\LABEQ{alfas}
\alfav=\left( \K(\X,\X) +\lambda\I{_{\n}} \right)\inv\yv,
\end{equation}
where $[\K(\X,\X)]_{r,s}=\k(\x(r),\x(s))$. %\notaRBT{PREGUNTA: DEBEMOS DEFINIR $\K(\X,\X)$? }
 Also, it is important to remark that a SRKHS is a particular case of the solution in \EQ{MOL}, even if the kernel is complex-valued. %If a complex-valued kernel is used, in the MOL it corresponds to the case $\kv\rrrr=\kv\jjjj$ and $\kv\rrjj=-\kv\jjrr$. 
 By analogy with covariances of complex-valued random variables, this formulation resembles the \emph{proper} case. In the proper case real and imaginary parts exhibit the same covariance while the covariance of the real part to the imaginary part is minus the covariance of the imaginary part to the real one {\cite{Schreier06}}. 
%At this point it is important to remark that all complex-valued RKHS of the form in \EQ{SRKHS} corresponds to a \emph{proper}-like solution. SRKHS are a particular case of the full MOL solution in \EQ{SRKHSri}, even if the kernel is complex-valued.
%
%, useful to solve, among others, the equalization problem. We next propose proper complex Gaussian processes for regression to learn with this SRKHS.

% Comentar cómo se queda en caso real e imag independiente.

% Comentar qué pasa si además tienen distintas varianzas.

\subsection{Connection to previous approaches}
%To work with complex-valued functions, not with real and imaginary parts separately, in \cite{Aronszajn50} it is proposed to build a complex Hilbert space considering functions of the form $\f(\cdot)=\f_1(\cdot)+\j \f_2(\cdot)$
%\begin{equation}
%f(\cdot)=\sum_{i=1}^{m}\alpha_ik_1(\cdot,\x({i}))+\j\sum_{i=1}^{m}\beta_ik_2(\cdot,\x({i}))
%\end{equation}

%. 

It is straightforward to show that previous approaches in \cite{OgunfunmiP11,Bouboulis11,Tobar12} belong to the SRKHS type with kernels as described in the following section and, therefore, are limited compared to the WRKHS.
An interesting singular case of this SRKHS formulation corresponds to the scenario where the real and imaginary parts are not related. In this case, 
%a solution $\krj\rrjj(\x\new,\x)=\krj\jjrr(\x\new,\x)=0$, 
$\kv\jjrr=-\kv\rrjj=0$ and the formulation yields, % that of the complexification formula,
\begin{equation}\LABEQ{f1f2}
\begin{bmatrix}f\rr(\x\new)\\f\jj(\x\new) \end{bmatrix}= 
\begin{bmatrix}
\kv\rrrr(\x\new,\X)\alfav\rr
\\
\kv\rrrr(\x\new,\X)\alfav\jj
 \end{bmatrix},
 \end{equation}
where the kernel in \EQ{SRKHS} is real.  This is the simple \emph{complexification} described in \cite{Aronszajn50,Paulsen09} based on building a complex Hilbert space considering functions of the form $\f(\cdot)=\f_r(\cdot)+\j \f_j(\cdot)$ where $\f\rr$ and $\f\jj$ are in class $\mathcal{F}$. 
%It is immediate to show that this procedure yields
%\begin{equation}\LABEQ{f1f2}
%f\rr(\x\new)+\j f\jj(\x\new) = 
%\mathbf{k}(\x\new,\X)\alfav\rr +\j
%\mathbf{k}(\x\new,\X)\alfav\jj
% \end{equation}
%%\begin{equation}
%%f(\cdot)=\sum_{i=1}^{m}\alpha_ik_1(\cdot,\x({i}))+\j\sum_{i=1}^{m}\beta_ik_2(\cdot,\x({i}))
%%\end{equation}
%where $\f_r$ and $\f_j$ are real and in class $\mathcal{F}$, $\k: \mathcal{X}\times \mathcal{X} \rightarrow \RN$, $\mathcal{X}\in\CN^d$. 
Note that $\f\rr$ and $\f\jj$ must be real for every input, and that being in the same class implies the same real-valued reproducing kernel for the real and the imaginary parts \cite{Aronszajn50}. {This procedure is limited in that it amounts to learning the real and imaginary parts independently but with the same kernel}.

%At this point we have a full interpretation of equation \EQ{SRKHS}.

%, as later shown in this section. 
%To endow the RKHS with more flexibility we could extend \EQ{rkhs} to the complex case as follows. %We focus on the following definition as an extension of \EQ{rkhs} to complex-valued reproducing kernels.

%\notajj{Previous proposition indicates that a SRKHS is limited with respect to an approach where real and imaginary parts are dealt with in a separated or extended formulation. }

%%%%%%%%%%%%%%%%%%%%%%%%%%%%%%%%%%%%%%%
%%%%%%%%%%%%%%%%%%%%%%%%%%%%%%%%%%%%%%%

\section{Kernel Design}\LABSEC{Ker}
The kernel is a key tool in RKHS: it encodes our assumptions about the function that we wish to learn. The kernel measures \emph{similarity} between inputs.
%, i.e., training points that are near to a test point should be informative about the prediction at that point. 
%There are some well-known examples of kernels used for real-valued applications, such as the squared exponential or Gaussian kernel and the inhomogeneous polynomial kernel, among others. 
In this section we first analyze kernels, including previous proposals. Then we face the design of pseudo-kernels, for WRKHS.

%%%%%%%%%%%%%%%%%%%%%%%%%%%%%%%%%%%%%%%%%%
{\subsection{Kernel}}

In \cite{Steinwart06,Bouboulis11} a complex-valued Gaussian kernel approach is proposed. %This method is a RKHS that can be written as a SRKHS in \EQ{SRKHS}. 
The kernel used was as an extension of the real Gaussian kernel: 
\begin{align}\LABEQ{RK}
\k_{\CN}(\x,\x')=\exp\left(-({\x-\x'^{*})}^\top(\x-\x'^{*})/\gamma\right).
\end{align}
If we separate the real and imaginary parts, $\x=\x\rr+\j\x\jj$ and $\x'=\x'\rr+\j\x'\jj$, then it follows that
\begin{align}\LABEQ{CGK}
\k_{\CN}(\x,\x')&=\exp\left(-(|\x\rr-\x\rr'|^2-|\x\jj+\x\jj'|^2)/\gamma\right)\nonumber\\
&\cdot\exp\left(-2\j(\x\rr-\x\rr')\trs(\x\jj+\x\jj')/\gamma\right)\nonumber\\
%&=\exp(-(|\x\rr-\x\rr'|^2-|\x\jj+\x\jj'|^2)/\gamma)\nonumber\\
&=\exp\left(-|\x\rr-\x\rr'|^2/\gamma\right)\exp\left(|\x\jj+\x\jj'|^2/\gamma\right)\nonumber\\
&\cdot\left(\cos(2(\x\rr-\x\rr')\trs(\x\jj+\x\jj')/\gamma)\right.\nonumber\\&\left.-\j\sin(2(\x\rr-\x\rr')\trs(\x\jj+\x\jj')/\gamma)\right),
%=&(\x\rr-\x\rr')\trs(\x\rr-\x\rr')-(\x\jj+\x\jj')\trs(\x\jj+\x\jj')\nonumber\\&+2\j(\x\rr-\x\rr')\trs
\end{align}
where $|\cdot|$ is the $\ell^2$-norm. 

Another complex-valued kernel was proposed in \cite{Tobar12} also within a SRKHS. The authors of the proposal remarked that the kernel in \EQ{RK} does not have the intuitive physical meaning of a measure of similarity of the samples and propose the so-called independent kernel: 
\begin{align}\LABEQ{indepCGK}
\k_{\textrm{ind}}(\x,\x')&=\kappa_{\Rext}\left(\x\rr,\x'\rr\right)+\kappa_{\Rext}\left(\x\jj,\x'\jj\right)\\&+\j\left(\kappa_{\Rext}\left(\x\rr,\x'\jj\right)-\kappa_{\Rext}\left(\x\jj,\x'\rr\right)\right),
\end{align}
where $\kappa_{\Rext}$ is a real kernel of real inputs.

These two kernels, $\k_\CN$ and $\k_\textrm{ind}$, were introduced in \cite{Bouboulis11} and \cite{Tobar12} as part of a machine of the type in \EQ{SRKHS}.  First conclusion, in the view of \PRP{SRKHS}, is that these methods belong to SRKHS. %\notaRBT{as they cannot provide the same solutions than the MOL in \EQ{MOL}, since 
For SRKHS $\kv\rrrr(\x,\x')=\kv\jjjj(\x,\x')$ the pseudo-kernel cancels and the kernel matrix is limited to have a particular symmetry. In the view of \PRP{SRKHS2}, both kernels $\k_\CN$ and $\k_\textrm{ind}$ are of the form given in \EQ{PK}, $k(\x,\x')=k\rrrr(\x,\x')-\j\k\rrjj(\x,\x')$, with $k\rrrr(\x,\x')=k\rrrr(\x',\x)$ and $\k\rrjj(\x,\x')=-\k\rrjj(\x',\x)$. %, where $k\rrrr(\x\new,\x({i}))=k\jjjj(\x\new,\x({i}))$ and $k\jjrr(\x\new,\x({i}))=-k\rrjj(\x\new,\x({i}))$.
%In matrix form and for the training set, it follows that
%\begin{equation}\LABEQ{complexProperkernel}
%\K=%\K\rr+\j\K\jj=
%\K\rrrr-\j\K\rrjj.
%\end{equation}
%
%\tacha{In the framework of a SRKHS the similarity between the real part for a pair of given inputs is provided by the real part of this kernel, $k\rrrr$. For the imaginary parts of a couple of inputs the similarity is given also by $k\rrrr$. The similarity between the real part for a given input and the imaginary part of another one is provided by the imaginary part of the kernel, $\k\rrjj$. Moreover, from \EQ{k12}-\EQ{k12b} we conclude how the real part is related to the imaginary part, given $\k\rrjj(\x,\x')$ we have $\k\jjrr(\x,\x')=-\k\rrjj(\x,\x')$. It is important to remark that this condition enforces some particular structures or symmetries in the complex part of the kernel. Since $\k\jjrr(\x,\x')=\k\rrjj(\x,\x')\trs$ then $\k\jjrr(\x,\x')\trs=-\k\rrjj(\x,\x')$, i.e., the kernel is skew-symmetric.} 
{These symmetries for one-dimensional complex-valued inputs are illustrated in \FIG{figkernelBoub} and \FIG{figkernelTobar}. In these figures it can be observed the particular way these kernels measure \emph{similarity} between inputs}. The kernel $\k_\CN$ in \EQ{RK}, measures similarities between real parts of the inputs with $|\x\rr-\x\rr'|^2$,} while for the imaginary ones it uses $|\x\jj+\x\jj'|^2$. Also, it is not stationary and has an oscillatory behavior. We illustrate these features in the example in Fig. \ref{fig:figkernelBoub}. The exponent in the kernel may easily grow large and positive as can be observed in the example depicted in Fig. \ref{fig:figkernelBoub3}. This might cause numerical problems in the learning algorithms. %, as later discussed in the Experiments Section.  
{On the other hand, the kernel $\k_\textrm{ind}$ in \EQ{indepCGK} has a very particular structure since it follows the structure in \EQ{PK} but it is not written as a function of the complex-valued inputs, but as a function of the real and imaginary parts of the inputs.} This way of measuring the similarity between the  inputs produces a particular {\it cross}-shape, as shown in the example in Fig. \ref{fig:figkernelTobar}. %Note that this kernel gives rise to a skew-symmetric cross-covariance matrix $\K\rrjj$. %In the example in Fig. \ref{fig:figkernelTobar}, the real Gaussian kernel is used in \EQ{indepCGK} as it was proposed in \cite{Tobar12}. 
Again, notice that because of the high constant values along the real and imaginary axis this kernel may be not useful for a wide range of systems. 
%the properties of this kernel may not be very useful when modeling the underlying statistical properties of a wide range of physical systems because of the high constant values along the real and imaginary axis. 

 %Unless the skew-symmetry of the real-imaginary parts of the outputs exists and is known for the problem at hand we should avoid the complex part of the kernel in the SRKHS. 

\begin{figure}[htb!]
\begin{center}
\includegraphics[width=8cm,draft=false]{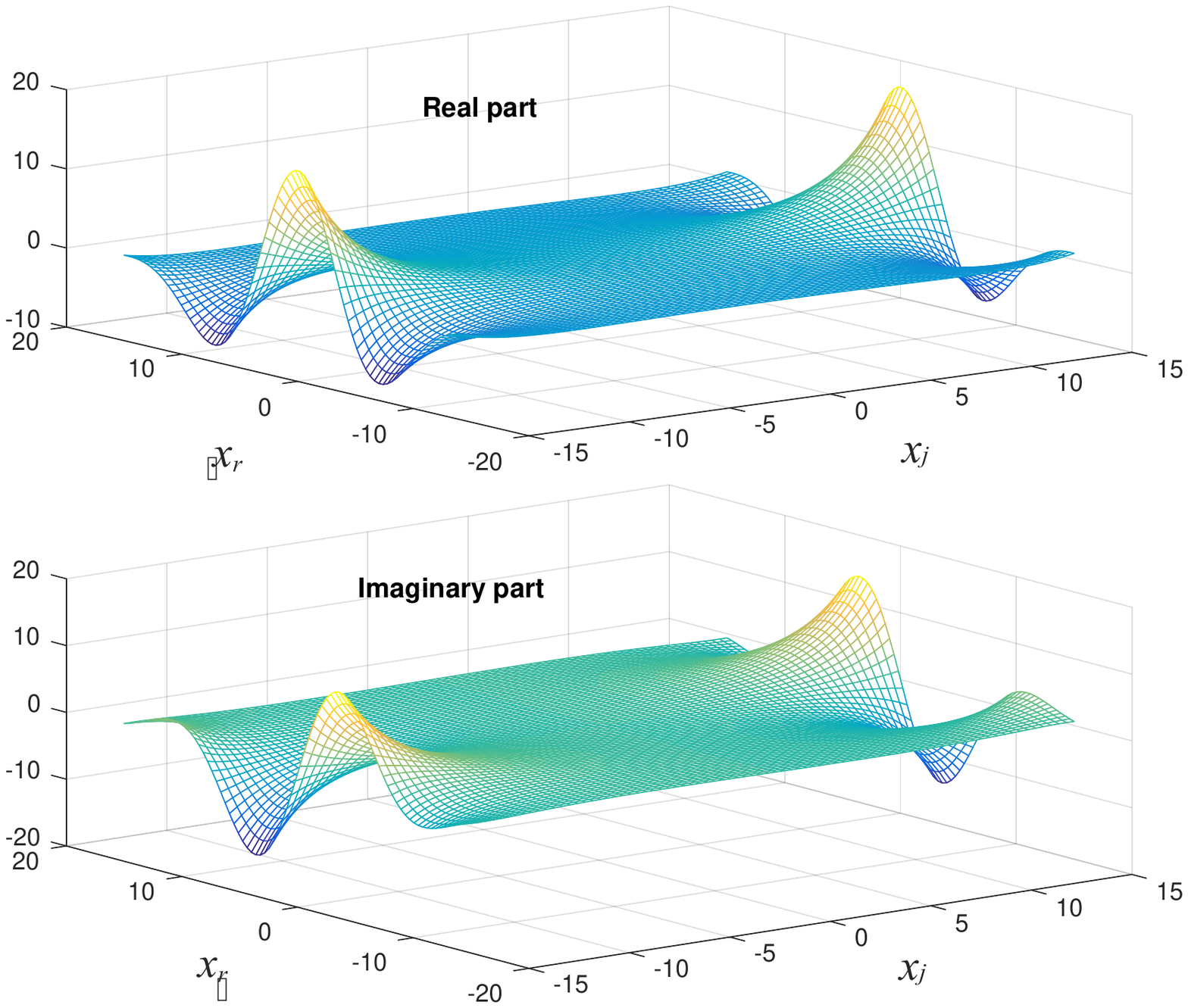}
\end{center}
%\vspace*{-.6cm}
\caption{Real and imaginary parts of $\k_{\CN}(x,x')$ when $x=0+\j0$ and $x'$ with real and imaginary parts in $[-15,15]$, $\gamma=80$.}
\LABFIG{figkernelBoub}
%load prueba_fer2_30_50e6_3_1_50
%\vspace*{-.3cm}
\end{figure}
%\begin{figure}[tb!]
%\begin{center}
%\includegraphics[width=8cm,draft=false]{figures/BoubTodosvsceroIm}
%\end{center}
%%\vspace*{-.6cm}
%\caption{Imaginary part of $\k_{\CN}(x,x')$ when $x=0+\j0$ and $x'$ with real and imaginary parts in $[-15,15]$, $\gamma=80$.}
%\LABFIG{figkernelBoub2}
%%load prueba_fer2_30_50e6_3_1_50
%%\vspace*{-.3cm}
%\end{figure}

\begin{figure}[htb!]
\begin{center}
\includegraphics[width=8cm,draft=false]{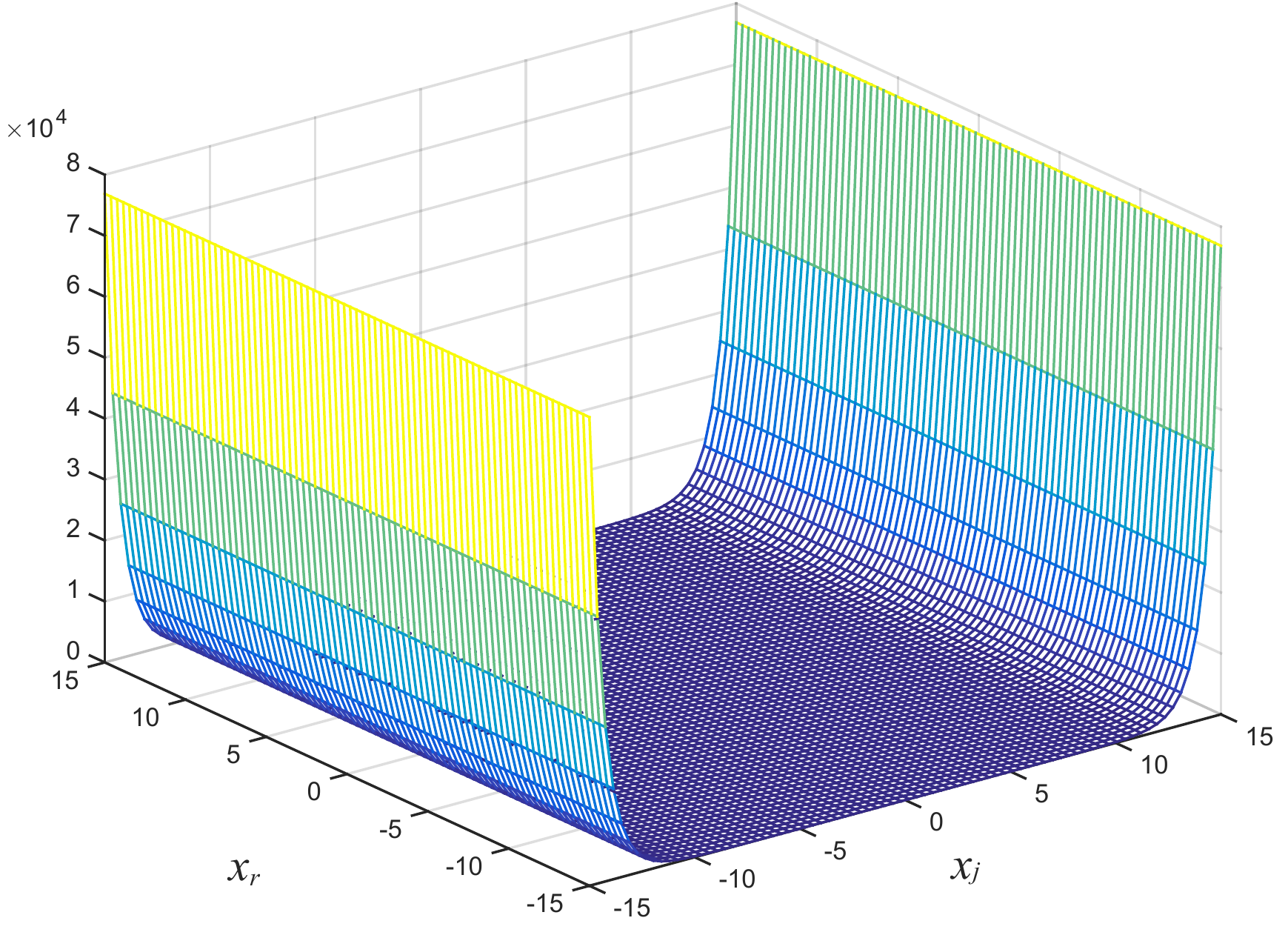}
\end{center}
%\vspace*{-.6cm}
\caption{$\k_{\CN}(x,x')$ when $x=x'$ with real and imaginary parts in $[-15,15]$, $\gamma=80$.}
\LABFIG{figkernelBoub3}
%load prueba_fer2_30_50e6_3_1_50
%\vspace*{-.3cm}
\end{figure}

\begin{figure}[htb!]
\begin{center}
\includegraphics[width=8cm,draft=false]{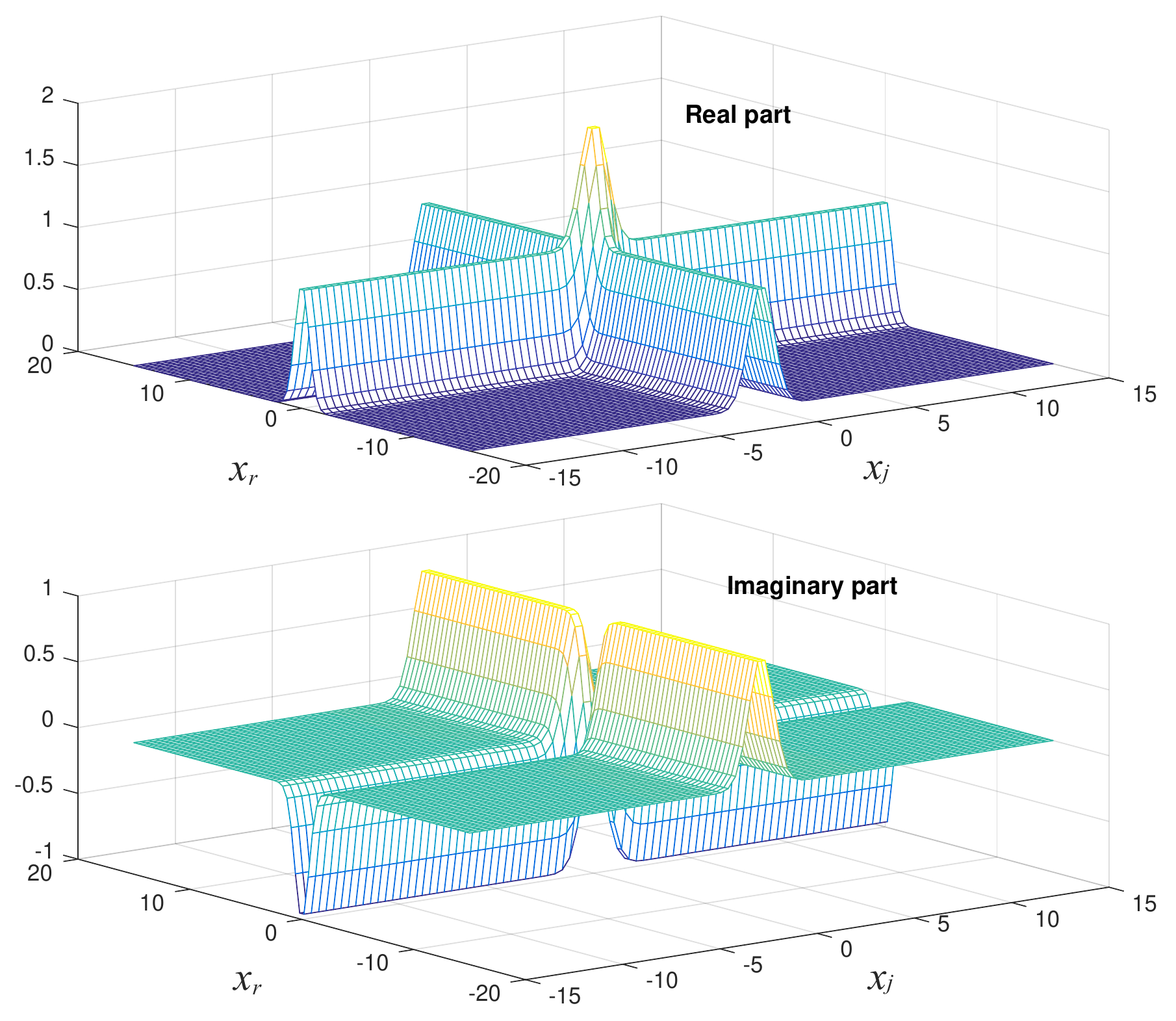}
\end{center}
%\vspace*{-.6cm}
\caption{Real and imaginary parts of $\k_\textrm{ind}(x,x')$ when $x=0+\j0$ and $x'$ with real and imaginary parts in $[-15,15]$, $\gamma=0.8$.}
\LABFIG{figkernelTobar}
%load prueba_fer2_30_50e6_3_1_50
%\vspace*{-.3cm}
\end{figure}
%\begin{figure}[tb!]
%\begin{center}
%\includegraphics[width=8cm,draft=false]{figures/TobarTodosvsceroImag}
%\end{center}
%%\vspace*{-.6cm}
%\caption{Imaginary part of $\k_{ind}(x,x')$ when $x=0+\j0$ and $x'$ with real and imaginary parts in $[-15,15]$, $\gamma=0.8$.}
%\LABFIG{figkernelTobar2}
%%load prueba_fer2_30_50e6_3_1_50
%%\vspace*{-.3cm}
%\end{figure}

%, in order to check its suitability to the channel equalization task proposed later in the Experiments section.

%Some kernels have been already proposed in a framework corresponding to the SRKHS. Main result in \PRP{SRKHS} states that if we use a SRKHS the kernel is of the form given in \EQ{PK},
%$k(\x\new,\x({i}))=k\rrrr(\x\new,\x({i}))-\j\k\rrjj(\x\new,\x({i}))$. 
%In matrix form and for the training set, it follows that
%\begin{equation}\LABEQ{complexProperkernel}
%\K=\K\rr+\j\K\jj=\K\rrrr-\j\K\rrjj.
%\end{equation}

%%%%%%%%%%%%%%%%%%%%%%%%%%%%%%%%%%%%%%%%%%
%{\subsection{Design of Kernels}}

Our conclusion is that, in WRKHS with null pseudo-kernel, enforcing a complex-valued kernel is counterproductive unless you identify, for the particular problem at hand, a skew-symmetry of the kind {$\k\rrjj(\x,\x')=-\k\rrjj(\x',\x)$. 
%Therefore, if that is not the case or if the problem at hand is better modeled with a null imaginary part for the kernel, $k\rrjj(\x',\x)=0$, we propose a real-valued kernel for SRKHS.} 
Note that a null pseudo-kernel and a real-valued kernel yields the complexification case in \EQ{f1f2}. % unless some specific skew-symmetry is known to be present in the system. 
{The way that similarity is measured and the structure of the kernel function are two important issues to take into account when designing the kernel. Regarding similarity, since we are working in a model with complex-valued inputs, we propose using the difference between complex-valued inputs, $\mathbf{d}_\x=(\x-\x')$. In addition, if an isotropic behavior is desired, functions $k\rrrr(\x,\x')$ and $k\rrjj(\x',\x)$ in \EQ{PK} could better rely on the inner product $\mathbf{d}_\x\her\mathbf{d}_\x=(\x-\x')\her(\x-\x')$ rather than on expressions of real and imaginary parts. 
%
%In SRKHS it has been discussed that in the general case where a skew-symmetry is not present, a real valued kernel is preferred, i.e., $k\rrjj(\x',\x)=0$. When the cross-covariance between real and imaginary parts of the outputs is null, the formulation yields that of the complexification formula \EQ{f1f2}. {If isotropy is also a desired property to model, 
We propose as an isotropic stationary kernel the following adaptation of the real-valued Gaussian kernel depending on the inner product $\mathbf{d}_\x\her\mathbf{d}_\x$:
\begin{equation} \LABEQ{expkernel}
k\rrrr(\x,\x')=\k_{G}(\x,\x')=\exp\left(-\mathbf{d}_\x\her\mathbf{d}_\x/\gamma\right).
\end{equation}
 An example of this kernel is shown in \FIG{figkernelUS}. We will use this kernel for the equalization problem in the Experiments Section, where usually real and imaginary parts in a digital communication constellations are independent but exhibit similar properties. % for scalar complex-valued inputs.  
\begin{figure}[tb!]
\begin{center}
\includegraphics[width=8cm,draft=false]{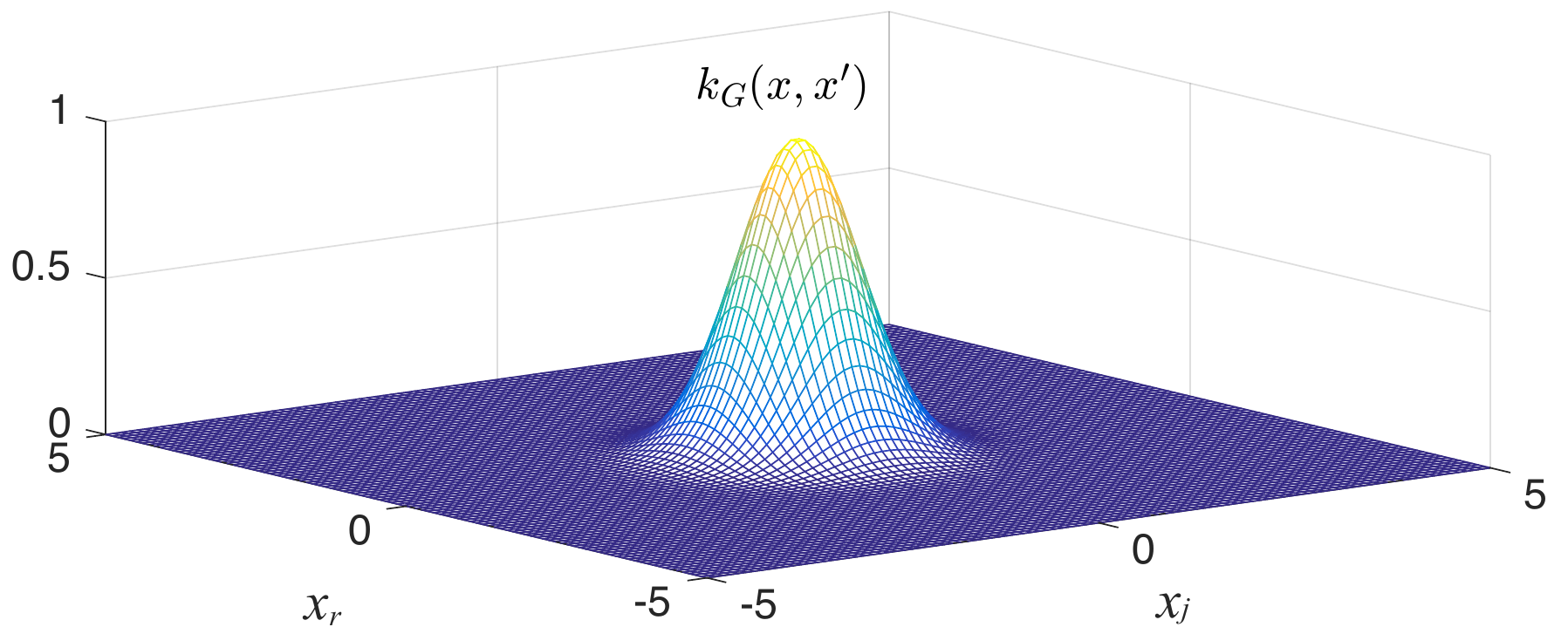}
\end{center}
%\vspace*{-.6cm}
\caption{Examples for $\k_{G}(x,x')$ with null imaginary part in \EQ{expkernel} when $x=0+\j0$ and $x'$ with real and imaginary parts in $[-5,5]$, $\gamma=0.8$.}
\LABFIG{figkernelUS}
%load prueba_fer2_30_50e6_3_1_50
%\vspace*{-.3cm}
\end{figure}
%\begin{figure}[tb!]
%\begin{center}
%\includegraphics[width=8cm,draft=false]{figures/UsselfIm}
%\end{center}
%%\vspace*{-.6cm}
%\caption{Imaginary part of $\k(x,x')$ in \EQ{kerfil} when $x=0+\j0$ and $x'$ with real and imaginary parts in $[-5,5]$, $\gamma=0.8$, $\mu=1-\j$.}
%\LABFIG{figkernelUS2}
%%load prueba_fer2_30_50e6_3_1_50
%%\vspace*{-.3cm}
%\end{figure}

\subsection{Kernel and Pseudo-Kernel}
In WRKHS we have both the kernel, {$\k(\x,\x')$}, and the pseudo-kernel, {$\pk(\x,\x')$. These kernels can be written as functions of the kernels of the real part, $\krj\rrrr(\x,\x')$, the imaginary part, $\krj\jjjj(\x,\x')$ and the real-imaginary parts $\krj\jjrr(\x,\x')$ and $\krj\rrjj(\x,\x')$, as in \EQ{covK}-\EQ{pcovK}}. Therefore the design is quite open. We bring here two particular but interesting cases. First one is the scenario where real and imaginary parts are independent but exhibit different properties and different kernels should be used. As a second case we design a kernel for the scenario where real and imaginary parts are not independent.

\subsubsection{ Different kernels for the real and imaginary parts}
If the real and imaginary parts need different kernels we may use a real-valued kernel for the real part, {$\krj\rrrr(\x,\x')$}, and another real-valued design for the imaginary one, {$\krj\jjjj(\x,\x')$}, assuming independence between real and imaginary parts. The kernels in \EQ{covK}-\EQ{pcovK} yield
{\begin{align}
\k(\x\new,\x({i}))&=\krj\rrrr(\x\new,\x({i}))+\krj\jjjj(\x\new,\x({i})),\nonumber\\
\pk(\x\new,\x({i}))&=\krj\rrrr(\x\new,\x({i}))-\krj\jjjj(\x\new,\x({i})).\LABEQ{Ks}
\end{align}}
%\begin{align}\LABEQ{Ks}
%{\K}(\X,\X)&=\K\rrrr(\X,\X)+ \K\jjjj(\X,\X)\nonumber \\%&+\j\left(\K\rrjj\trs(\X_\n,\X_\n)-\K\rrjj(\X_\n,\X_\n)\right),\LABEQ{covK}\\
%{\matr{\tilde{K}}}(\X,\X)&=\K\rrrr(\X,\X)- \K\jjjj(\X,\X). %\nonumber\\&+\j\left(\K\rrjj\trs(\X_\n,\X_\n)+\K\rrjj(\X_\n,\X_\n) \right).\LABEQ{pcovK}
%\end{align}
Note that this scenario is simple, but the SRKHS is not a valid framework to explain it in a complex-valued formalism and a pseudo-kernel is needed. Besides, the resulting kernels are real-valued.

\subsubsection{{Non-independent}  real and imaginary parts}
{This scenario} can be easily handled by using the concept of \emph{separable kernel} and \emph{sum of separable} (SoS) kernels \cite{Alvarez12}. In the complex case a mixed effect regularizer (MER) translates into 
%\begin{align}
%{\K}(\X_\n,\X_\n)&=2\sum_{q=1}^Q \K^{(q)} (\X_\n,\X_\n) %\nonumber \\
%%&+\j\omega^{(q)}\left(\K\rrjj\trs(\X_\n,\X_\n)-\K\rrjj(\X_\n,\X_\n)\right)
%,\LABEQ{Ksos}\\
%{\matr{\tilde{K}}}(\X,\X)&=2\j\sum_{q=1}^Q \omega^{(q)} \K^{(q)}(\X_\n,\X_\n), 
%\LABEQ{pKsos}
%\end{align}
{\begin{align}
{\k}(\x\new,\x(i))&=2\sum_{q=1}^Q \k^{(q)}(\x\new,\x(i)) %\nonumber \\
%&+\j\omega^{(q)}\left(\K\rrjj\trs(\X_\n,\X_\n)-\K\rrjj(\X_\n,\X_\n)\right)
,\LABEQ{Ksos}\\
{\pk}(\x\new,\x(i))&=2\j\sum_{q=1}^Q \omega^{(q)} \pk^{(q)}(\x\new,\x(i)), 
\LABEQ{pKsos}
\end{align}}
if we choose $\k^{(q)}(\x\new,\x(i))$ to be real-valued kernels of complex-valued inputs, where %$\k^{(q)}(\x,\x')=\k^{(q)}(\x',\x)$, and 
$0<\omega^{(q)}<1$. %Note that cluster based or graph regularizer would lead to a similar solution as the equivalent MOL model has dimension two. 

%%%%%%%%%%%%%%%%%%%%%%%%%%%%%%%%%%%%%%%
%%%%%%%%%%%%%%%%%%%%%%%%%%%%%%%%%%%%%%%
\section{experiments}\LABSEC{Exp}
\subsection{Learning with WRKHS}

To illustrate the learning with WRKHS we bring here two synthetic experiments. In the first one we learn with a different similarity measurement for the real and the imaginary parts. We use the WRKHS solution in \EQ{fWRKHS} with the kernel and pseudo-kernel in \EQ{Ks}. In the second scenario we exploit the relation between the real and imaginary parts of the output, using \EQ{Ksos}-\EQ{pKsos}.

\subsubsection{Real and imaginary parts}

We propose to learn a non-linear function of the type $y(x)=y\rr(x)+\j y\jj(x)$, where $x=x\rr+\j x\jj$ and in this experiment
\begin{align}
y\rr(x)=&\sum_{r=-1}^1\sinc(1.2 x\rr+2r)\cdot \sinc(1.2 x\jj-2r)\nonumber\\
y\jj(x)=&\sinc(0.2x\jj-1.5).
\end{align}
We generate $\n=200$ random training samples of $y(x)$ in the range $[-5,5]$. In \EQ{Ks}, $\k\rrrr(x,x')$ and $\k\jjjj(x,x')$ are $\k_{G}(x,x')$ in \EQ{expkernel} with $\gamma=1$ and $\gamma=3.5$, respectively. Since the imaginary part of the output has a softer behavior, the optimal hyperparameter of the kernel is 
larger than for the real part. The result is included in \FIG{WRKHSi}, where the training samples are plotted in red circles. Without a pseudo-kernel we can not use a different similarity measurement for the real and the imaginary parts. In \FIG{WRKHSni} we include the result of the learning of the imaginary part if the same kernel with $\gamma=1$ is used for both the real and the imaginary parts. In this case we observe that the learning of the imaginary part exhibits quite a larger error. The overall mean square error (MSE) with the WRKHS is $-54.9$ dB while WRKHS with null pseudo-kernel is $-38.8$ dB. 

\begin{figure}[tb!]
\begin{center}
\includegraphics[width=9cm]{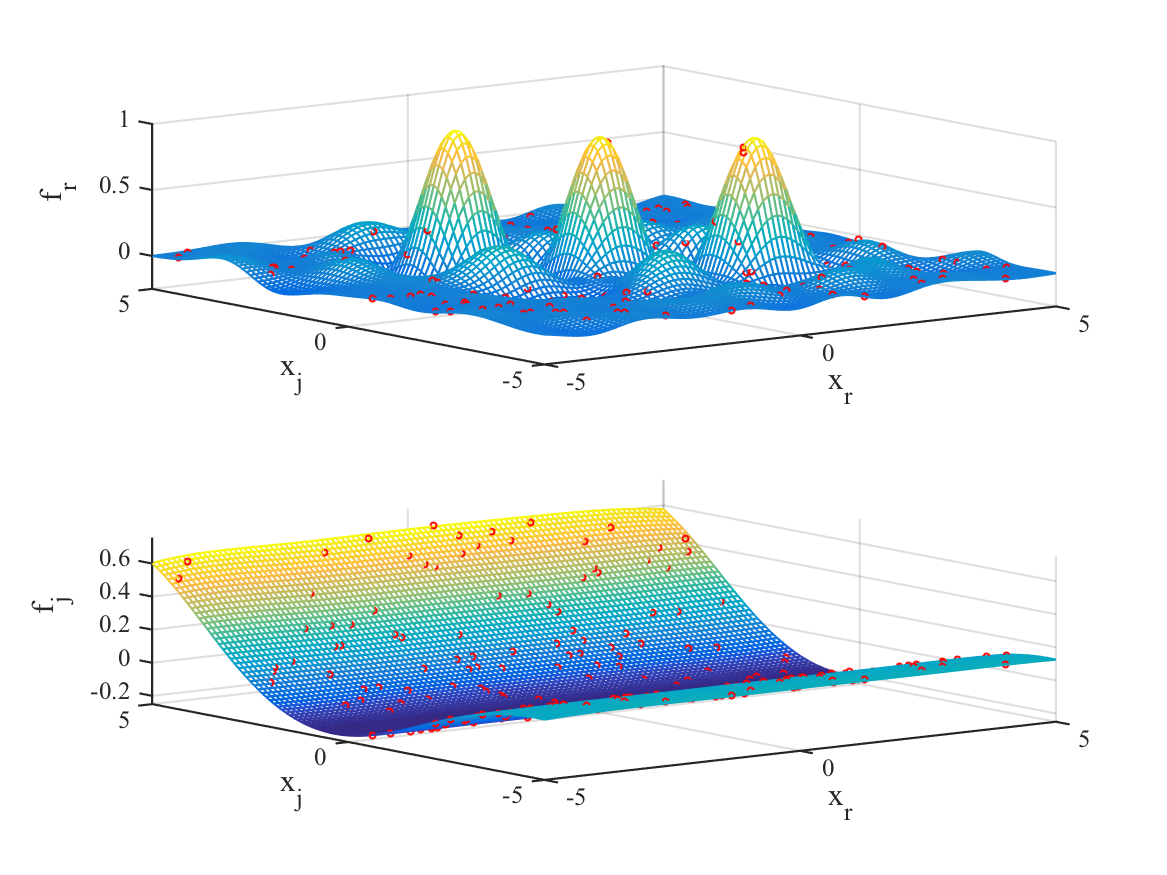}
\end{center}
\caption{Real (top) and imaginary (bottom) parts of the WRKHS estimation $\f(\x)$ versus the real and imaginary parts of the input. The training samples are depicted as red circles.}
\LABFIG{WRKHSi}
\end{figure}

\begin{figure}[tb!]
\begin{center}
\includegraphics[width=9cm,draft=false]{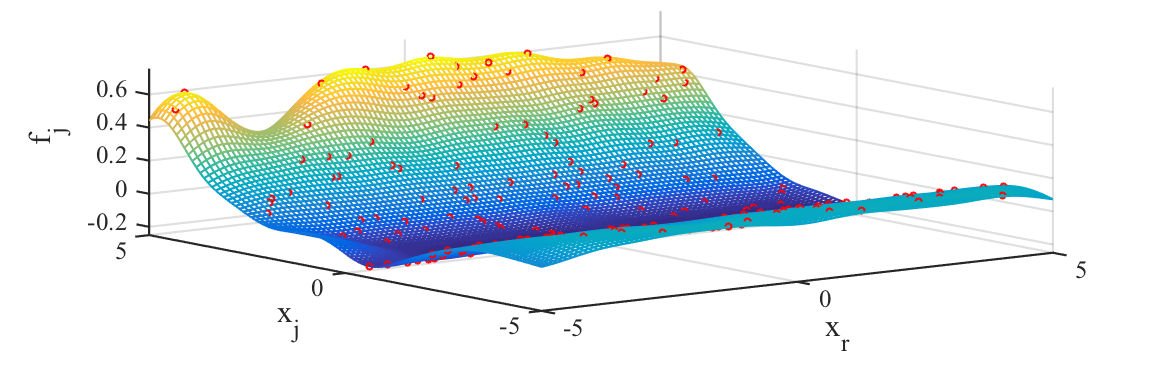}
\end{center}
\caption{Imaginary part of the WRKHS estimation $\f(\x)$ with null pseudo-kernel versus the real and imaginary parts of the input. 
Same kernel for the real and imaginary parts are used. 
The training samples are depicted as red circles. }
\LABFIG{WRKHSni}
\end{figure}

\subsubsection{{Non-independent}real and imaginary parts} 

We propose to learn a non-linear function of the type $y(x)=y\rr(x)+\j y\jj(x)$, where $y\rr(x)=z\rr + \omega z\jj$, $y\jj(x)=z\jj + \omega z\rr$, % and %in this experiment we first compute
\begin{align}
z\rr(x)=&\sinc(0.5 x\rr)\cdot \sinc(0.5 x\jj)\nonumber\\
z\jj(x)=&0.1\cdot \sinc(0.3x\jj).
\end{align}
and $x=x\rr+\j x\jj$.
We generate $\n=200$ random training samples of $y(x)$ in the range $[-5,5]$. In \EQ{Ks} we use MER and SoS but with just one term, $Q=1$, setting $\omega=\omega^{(1)}=0.3$. The kernel $\k^{(1)}(\x,\x')$ in \EQ{Ksos}-\EQ{pKsos} is the Gaussian, $\k_{G}(x,x')$, in \EQ{expkernel} with $\gamma=2$. The result of the learning is depicted in \FIG{WRKHSc}. Note that the real and imaginary parts are similar but not equal. The MSE of this solution is $-45.3$ dB, while the solution for $\omega^{(1)}=0$, that corresponds to SRKHS, is $-41.8$ dB.

\begin{figure}[tb!]
\begin{center}
\includegraphics[width=9cm]{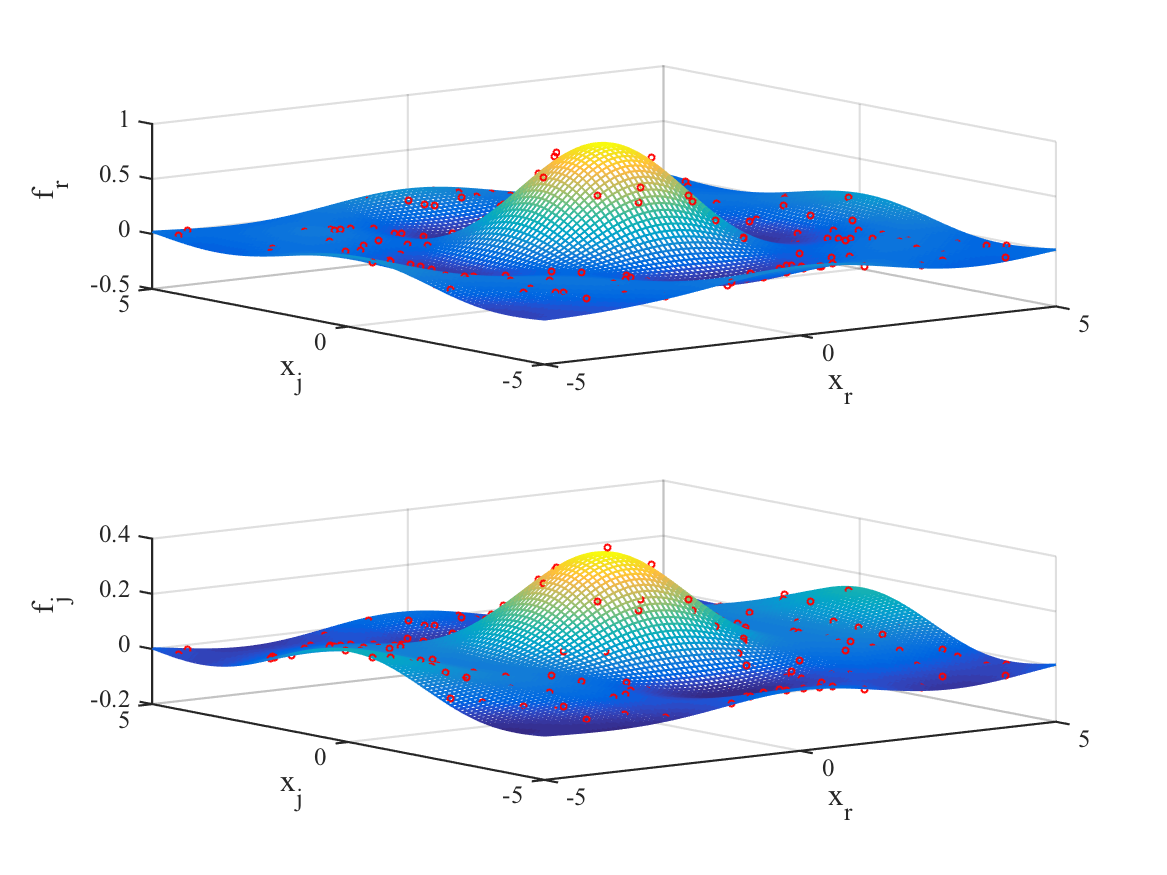}
\end{center}
\caption{Real (top) and imaginary (bottom) parts of the WRKHS estimation $\f(\x)$ versus the real and imaginary parts of the input. A separable kernel was used with $Q=1$ and $\omega^{(1)}=0.3$. The training samples are depicted as red circles.}
\LABFIG{WRKHSc}
\end{figure}

%HARD
\begin{figure}[tb!]
\begin{center}
\includegraphics[width=8.7cm, draft=false]{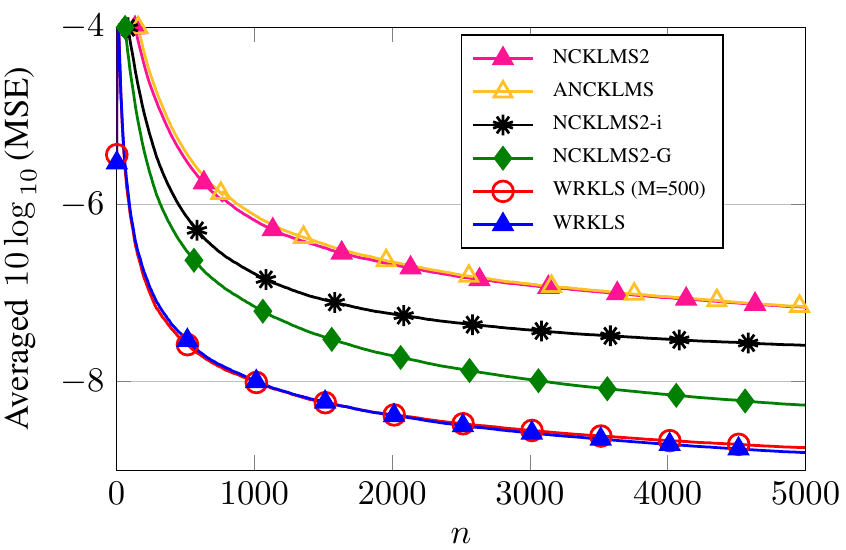}
\end{center}
\vspace*{-.6cm}
\caption{Averaged MSE along $\n$ for NCKLMS2, ANCKLMS, NCKLMS2-i, NCKLMS2-G, the WRKLS and the WRKLS with M=500 basis for the strong nonlinear channel equalization problem and the circular input case.}
%load prueba_fer2_30_50e6_3_1_50
%\vspace*{-.3cm}
\LABFIG{fig6}
\end{figure}
%

%HARD
\begin{figure}[tb!]
\begin{center}
\includegraphics[width=8.7cm, draft=false]{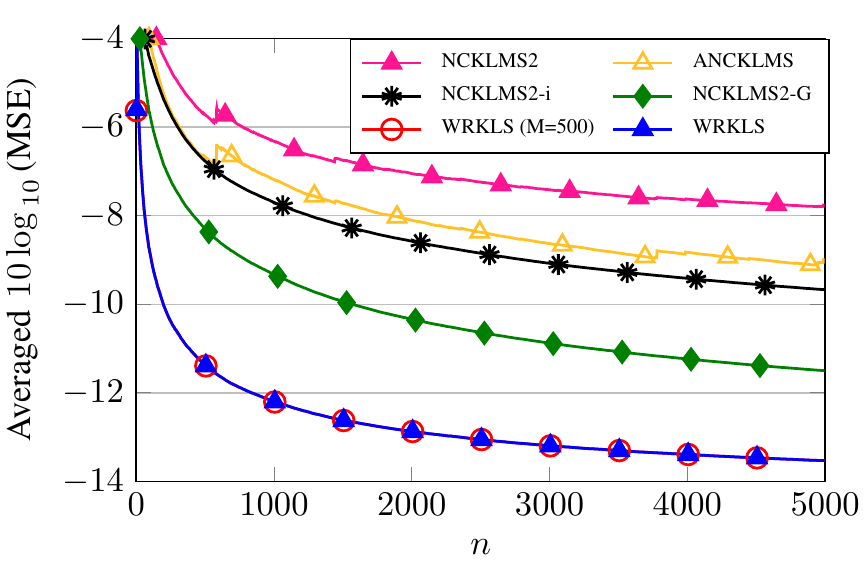}
\end{center}
\vspace*{-.6cm}
\caption{Averaged MSE along $\n$ for NCKLMS2, ANCKLMS, NCKLMS2-i, NCKLMS2-G, the WRKLS and the WRKLS with M=500 basis for the strong nonlinear channel equalization and the noncircular input case ($\rho=0.1$).}
\LABFIG{fig8}
%load prueba_fer2_30_50e6_3_1_50
%\vspace*{-.3cm}
\end{figure}

\subsection{Nonlinear channel equalization}
The main advantage of the analysis of the WRKHS in this paper is that it greatly facilitates decisions to make through the design to select the simplest model. %Our model selection is the simplest one explaining the system at hand. 
To illustrate this point we bring here the nonlinear channel equalization in \cite{Bouboulis11} and \cite{Bouboulis12}.
%
%The performance of the proposed proper complex GPR was tested in the context of the nonlinear channel equalization task in \cite{Bouboulis11} and \cite{Bouboulis12}. 
%
We use the channel considered in \cite{Bouboulis11} and \cite{Bouboulis12}. It consists of a linear filter $t(n)=(-0.9+0.8\j)\cdot s(n)+(0.6-0.7\j)\cdot s(n-1)$ and a memoryless nonlinearity. The nonlinearity is $q(n)=t(n)+(0.2+0.25\j)\cdot t^{2}(n)+ (0.12+0.09\j)\cdot t^{3}(n)$ and it is labeled as {\it strong nonlinear channel}. The input signals had the form $s(n)=0.70(\sqrt{1-\rho^{2}}X(n)+\j\rho Y(n))$,  and $X(n)$ and $Y(n)$ were Gaussian random variables. Note that the real and the imaginary parts of the input signals were generated independently. The input signals are circular for $\rho=1/\sqrt{2}$ and highly noncircular if $\rho$ approaches $0$ or $1$. At the receiver end of the channel, the signal $q(n)$ was corrupted by additive white circular Gaussian noise with the SNR set to 16 dB, as in \cite{Bouboulis11}.
The aim of the channel equalization task is to construct an inverse filter, which acts on the received signal $r(t)$ and reproduces the original input signal $s(n)$ as close as possible. To this end, the inputs to the equalizer are the sets of samples $\vect{x}=[r(n+D),r(n+D-1),\cdots,r(n+D-L+1)]^{\top}$, where $L>0$ is the filter length and $D$ is the equalization time delay. %These parameters were set to $L=5$ and $D=2$, as in 

The authors in \cite{Bouboulis11} first proposed a machine than can be easily proved to be within the SRKHS framework to later improve it adding more terms \cite{Bouboulis12}. Kernel used was $\k_\CN$ in \EQ{RK}. The algorithms proposed, the NCKLMS2 and the ACKLMS, are of the kernel LMS type and they use novelty sparsification criterion to reduce the number of training samples used to compute the solution. We use the code available in \cite{boubouliscode} to run these algorithms. All the parameters required for both algorithms ($\gamma$ in kernel and step update parameter) are set to the values described in \cite{Bouboulis11} and \cite{Bouboulis12}, except for the {\it strong nonlinear channel} in the noncircular case, where in order to ensure convergence it is needed to increase $\gamma$ to $\gamma=20^{2}$. The parameters for the sparsification are $\delta_{1}=0.15$ and $\delta_{2}=0.2$. 
Also, $L=5$ and $D=2$, on a set of 5000 samples of the input signal considering both the circular and the noncircular ($\rho=0.1$) cases and the described nonlinear channel. In all cases the results in \cite{Bouboulis11} and \cite{Bouboulis12} are averaged over 500 trials where the input signals $s(n)$ and noise are generated randomly. 

\subsubsection{Circular case, $\rho=1/\sqrt{2}$}

From the problem at hand we know that the output has independent real and imaginary parts. Since there is no relation between real and imaginary parts we have $\k\rrjj(\x,\x')=0$ and the kernel should be chosen real-valued. Hence, we avoid the design of a skew-symmetric imaginary part {$\k\rrjj(\x,\x')=-\k\rrjj(\x',\x)$}. Besides, using the same kernel for the real and the imaginary parts is a suitable assumption in the circular case, since we expect both to behave similarly in terms of similarity. We conclude that the solution lies within a WRKHS with null pseudo-kernel, i.e. {\EQ{SRKHS}}. The Gaussian kernel $k\rrrr(\x,\x')$$=\k_{G}(\x,\x')$ in \EQ{expkernel} with null imaginary part is a suitable solution in this equalization problem \cite{PerezCruz13gp,PerezCruz08}. 
The values for the hyperparameters are set to $\gamma=8.92$ and $\lambda=0.32$.
 %. gamma = 8,9168; sigma = 0,3242
To compare the solution in \EQ{SRKHS} with coefficients given by \EQ{alfas}, we derive the kernel recursive least square (RKLS) \cite{Vaerenbergh12} version of WRKHS with null pseudo-kernel, denoted as WRKLS. %We denote this approach as SRKLS to emphasize it is within a complex-valued SRKHS. 
This approach adaptively computes the solution for each new input and retains $M$ basis or samples to compute it, reducing the computational complexity.  
We first compare the performances of the NCKLMS2 \cite{Bouboulis11}, the ACKLMS \cite{Bouboulis12} and the SRKLS. %
 We include two versions of the WRKHS: with basis removal criterion ($M=500$) and without basis removal criterion. 
Note that the number of bases used by the NCKLMS2 and the ACKLMS algorithms with the novelty sparsification criterion grows above $2000$ in the experiments in this section, and therefore the choice of $M=500$ is far below that number. 
In Figs. \ref{fig:fig6} we can observe the averaged MSE along the input samples for these four methods and the channels and inputs described. The MSE value depicted for each sample is the averaged MSE for all previous outputs, as in \cite{Bouboulis11,boubouliscode}.  

\subsubsection{Non-circular case, $\rho=0.1$}
In the non-circular case we have a different behavior for the real and imaginary part of the output, but just in the scaling. The similarity between inputs remains the same and we may use the same kernel for the real and imaginary part. The solution for the coefficients will change, but just in scale. The module of $\alfav\rr$ in \EQ{alfas} increases while the module of $\alfav\jj$ decreases. Therefore, we may apply exactly the same solution as in the circular case. In Figs. \ref{fig:fig8} we compare the performance of the WKRLS with no basis selection and $500$ basis to the results of the mean square error for the NCKLMS2 and the ACKLMS with sparsification. The values for the hyperparameters are set to $\gamma=10.4$ and $\lambda=0.18$.

\subsubsection{Discussion}
It can be observed in the figures the remarkable good results of the WKRLS in all the cases,  {\it strong} nonlinear channels and circular or noncircular signals. 
With only $M=500$ bases used for the prediction, the averaged MSE is very close to the method using all samples.
When comparing with the NCKLMS2 and ACKLMS, the SKRLS remarkably outperforms both algorithms that use above $2000$ basis. %, %Another important issue is the variance of the estimation. We found some convergence problems in the learning process of both the NCKLMS2 and ACKLMS algorithms for some cases. These problems can not be observed in the previous figures because, as in \cite{Bouboulis11}, the MSE depicted at some time is the MSE averaged for the estimated output at this time and all the previous estimated outputs. When just the estimated MSE for every time {instant} is depicted, with no average with the previous results, these convergence problems are highlighted. To illustrate this issue, we include Figs. \ref{fig:fig82} and \ref{fig:fig83}, with the results of Figs. \ref{fig:fig8} with no average. In \FIG{fig82} we depict the results for the NCKLMS2 and ACKLMS algorithms. The performance of the SRKLS is included in \FIG{fig83}. Note that the NCKLMS2 and ACKLMS algorithms are not able to provide a good prediction for some outputs, with peak values for the MSE above $10$ dB. This is not the case for the SKRLS with kernel in \EQ{expkernel}, as can be observed in \FIG{fig83}. These convergence problems of the NCKLMS2 or ACKLMS algorithms are mainly due to the selection of the kernel. As already discussed, the kernel in \EQ{RK} is non-stationary nor isotropic. In particular, for some new inputs near a basis the kernel grows exponentially attending to the value of the input, as depicted in \FIG{figkernelBoub3}. This makes the convergence quite unstable. 
The gains are not only due to the better capabilities of the recursive and basis removal approach used but to the model selection. From the results in this paper we conclude first that the best option is a WKRHS with real-valued kernel and null pseudo-kernel. Since real and imaginary parts are independent and they exhibit similar similitude measure up to a scaling. Hence, a good selection of kernel and (null) pseudo-kernel improves the final results.  
To further illustrate this point, we also include the NCKLMS2 algorithm with the kernel used in the WRKLS, in \EQ{expkernel}. This algorithm is labeled as NCKLMS2-G in the {figures}. The parameters for this algorithm are the same that were previously used for the NCKLMS2, and the novelty criterion is again used for the sparsification. 
Note that by using the proposed kernel we obtain a much better performance in all cases when comparing with the NCKLMS2 or ACKLMS algorithms. %, although not as good as the one of the WRKLS. 
%As expected, instabilities vanish, as shown in \FIG{fig83}.
%
Finally, for the sake of completeness we also include in the comparison the method in \cite{Tobar12}: the NCKLMS2 algorithm with the independent kernel (\ref{eq:indepCGK}) with $\kappa_{\Rext}$ being the real-valued Gaussian kernel. We labeled this algorithm as NCKLMS2-i in the figures. Although this kernel seems more suitable for the problem at hand than the complex Gaussian kernel \EQ{RK}, as shown in the figures, the performance is not as good as the NCKLMS2-G algorithm. The reason, again, is a sub-optimal model selection, where we have a kernel with non-null imaginary part and with the particular {\it cross}-shape shown in Fig. \ref{fig:figkernelTobar}.

\section{Conclusions}

Complex-valued kernel regression has been tackled by adapting the real-valued approaches in a straight forward manner \cite{OgunfunmiP11,Bouboulis11,Tobar12}. As in strictly linear estimation, this is useful in many scenarios, it is not efficient in others. %For example, if real and imaginary parts are independent and different in terms of similarity measurement or if they are correlated. 
We develop a novel solution, WRKHS, to avoid this limitation. The solution is based on including a pseudo-kernel. The resulting structure of the regressor resembles that of the widely complex linear solutions, being capable of learning the real and imaginary parts of the output regardless of the relation between them.
In the experiments we show how systems with independent and different real and imaginary parts are better learned. Regression for correlated real and imaginary parts is also improved. We introduce some proposals for the kernel and the pseudo-kernel. The complex-valued nature of the kernel and the pseudo-kernel is also discussed. When the pseudo-kernel cancels, special attention is to be paid to the imaginary part of the kernel. In this case the imaginary part must be skew symmetric or null. We apply these concepts to face the nonlinear channel equalization, minimizing the overall error. We believe the results in this paper are relevant to better face any complex-valued regression problem, using complex-valued formulation.

\ifCLASSOPTIONcaptionsoff
  \newpage
\fi



%\bibliographystyle{IEEEtran}

\bibliographystyle{IEEEtran}
\bibliography{IEEEabrv,CGPR,murilloGP,SSCDMA}

%\begin{thebibliography}{1}
%\end{thebibliography}


%\end{IEEEbiography}


%\newpage



%\vfill


\end{document}